\documentclass[10pt]{article} 
\usepackage{arxiv}
 
\usepackage{natbib}
\bibpunct{(}{)}{;}{a}{,}{,}
\usepackage{caption}

\usepackage{macros}

\theoremstyle{plain}
\newtheorem{definition}{Definition}

\usepackage{tikz, pgfplots}
\usetikzlibrary{arrows.meta, calc, positioning, matrix}

\newcommand{\modelclass}{\mathcal{F}}
\newcommand{\domain}{\mathcal{X}}
\newcommand{\classes}{\mathcal{Y}}
\newcommand{\predset}{\widehat{\mathcal{Y}}}   
\newcommand{\reals}{\mathbb{R}}
\newcommand{\loss}{\ell}
\newcommand{\normal}{\mathcal{N}}     
\newcommand{\Xcf}{X^{\text{CF}}}      

\DeclareMathOperator*{\E}{\mathbb{E}} 

\title{The Risks of Recourse in Binary Classification}

\author{%
    \bf{Hidde Fokkema} \\
    Korteweg-de Vries Institute for Mathematics\\
    University of Amsterdam\\
    \texttt{h.j.fokkema@uva.nl} \\
    \And
    \bf{Damien Garreau} \\
    CAIDAS\\
    University of W\"urzburg \\
    \texttt{damien.garreau@uni-wuerzburg.de} \\
    \AND
    \bf{Tim van Erven} \\
    Korteweg-de Vries Institute for Mathematics\\
    University of Amsterdam\\
    \texttt{tim@timvanerven.nl}
}

\begin{document}
\maketitle

\begin{abstract}
  Algorithmic recourse provides explanations that help users overturn an
  unfavorable decision by a machine learning system. But so far very
  little attention has been paid to whether providing recourse is
  beneficial or not. We introduce an abstract learning-theoretic framework
  that compares the risks (i.e.,\ expected losses) for classification with
  and without algorithmic recourse. This allows us to answer the question
  of when providing recourse is beneficial or harmful at the population
  level. Surprisingly, we find that there are many plausible scenarios in
  which providing recourse turns out to be harmful, because it pushes
  users to regions of higher class uncertainty and therefore leads to more
  mistakes. We further study whether the party deploying the classifier
  has an incentive to strategize in anticipation of having to provide
  recourse, and we find that sometimes they do, to the detriment of their
  users. Providing algorithmic recourse may therefore also be harmful at
  the systemic level. We confirm our theoretical findings in experiments
  on simulated and real-world data. All in all, we conclude that the
  current concept of algorithmic recourse is not reliably beneficial, and
  therefore requires rethinking.
\end{abstract}

\section{Introduction}\label{sec:introduction}
Machine learning (ML) models are increasingly being used to make consequential
decisions in areas such as finance \citep{mukerjee2002multi}, healthcare
\citep{begoli2019need,grote2020ethics}, and hiring
\citep{nabi2018fair,schumann2020we}.  When these decisions are unfavorable to
the people they affect, algorithmic recourse provides explanations and
recommendations to favorably change their situation \citep{karimi2022survey}.
For instance, when an individual is denied a bank loan, they might like to know
the reasons and in particular what they can do to get a loan in the future.

A prominent approach to providing recourse is via counterfactual
explanations, which suggest how the individual should change their
features in order to flip the decision of the ML model
\citep{wachter2017counterfactual,ustun2019actionable,joshi2019towards}.
Originally, counterfactuals were chosen to minimize the distance between
the original and the new features \citep{wachter2017counterfactual}, but
more recently attention has also been paid to generating realistic
suggestions which are actionable and lie on the data manifold
\citep{ustun2019actionable,joshi2019towards}. 
In addition, various types
of robustness have been studied, including to random perturbations
\citep{virgolin2023cogs,dominguez2022robust,
pawelczyk2022probabilistically}, to data shifts
\citep{rawal2020algorithmic, dutta2022tree}, or to the case that the
counterfactual might not be perfectly implementable
\citep{artelt2021evaluating}. It has further been recognized that
providing recourse has consequences at the population level, because it
changes the distribution of the data. These consequences have been
studied in the context of fairness for subgroups \citep{gupta2019equalizing} and with
respect to social segregation \citep{gao2023impact}, but so far there
has been no work that studies the consequences of providing recourse for
classification accuracy.

To see why accuracy matters, consider again the loan example mentioned
above. If a person is able to repay a loan they got through recourse,
then recourse has been beneficial. But if they end up defaulting on
their payment, then recourse has actually been harmful, both for the
user and the lending institution. Providing recourse in a way that
undermines the accuracy of the ML model in determining which users are
likely to default, can therefore be dangerous. In fact, the bank loan
example above, which is standard in the recourse literature, is also
used as a motivating example in the context of strategic classification.
There, it is seen as a significant risk that loan applicants
might try to game the system by changing their features to flip the
class without actually improving their true creditworthiness
\citep{brown22stateful,performative2020hardt, smitha2019socialcost}.
\begin{figure}[t!]
    \vspace{0.5cm}
    \centering
    \def\svgwidth{\columnwidth}
    \resizebox{0.9\textwidth}{!}{
\begingroup%
  \makeatletter%
  \providecommand\color[2][]{%
    \errmessage{(Inkscape) Color is used for the text in Inkscape, but the package 'color.sty' is not loaded}%
    \renewcommand\color[2][]{}%
  }%
  \providecommand\transparent[1]{%
    \errmessage{(Inkscape) Transparency is used (non-zero) for the text in Inkscape, but the package 'transparent.sty' is not loaded}%
    \renewcommand\transparent[1]{}%
  }%
  \providecommand\rotatebox[2]{#2}%
  \newcommand*\fsize{\dimexpr\f@size pt\relax}%
  \newcommand*\lineheight[1]{\fontsize{\fsize}{#1\fsize}\selectfont}%
  \ifx\svgwidth\undefined%
    \setlength{\unitlength}{205.9201607bp}%
    \ifx\svgscale\undefined%
      \relax%
    \else%
      \setlength{\unitlength}{\unitlength * \real{\svgscale}}%
    \fi%
  \else%
    \setlength{\unitlength}{\svgwidth}%
  \fi%
  \global\let\svgwidth\undefined%
  \global\let\svgscale\undefined%
  \makeatother%
  \begin{picture}(1,0.36949907)%
    \lineheight{1}%
    \setlength\tabcolsep{0pt}%
    \put(0,0){\includegraphics[width=\unitlength,page=1]{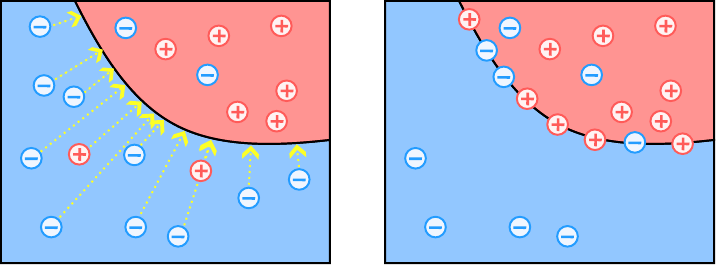}}%
    \put(0.16482317,0.38789916){\color[rgb]{0,0,0}\makebox(0,0)[lt]{\lineheight{0.75}\smash{\begin{tabular}[t]{l}\Large$R_P(f) =0.2$ \end{tabular}}}}%
    \put(0.40602499,0.06353111){\color[rgb]{0,0,0}\makebox(0,0)[lt]{\lineheight{0.75}\smash{\begin{tabular}[t]{l}\Large$x_0$ \end{tabular}}}}%
    \put(0.94660796,0.11353422){\color[rgb]{0,0,0}\makebox(0,0)[lt]{\lineheight{0.75}\smash{\begin{tabular}[t]{l}\Large$x$\end{tabular}}}}%
    \put(0.68896806,0.38789916){\color[rgb]{0,0,0}\makebox(0,0)[lt]{\lineheight{0.75}\smash{\begin{tabular}[t]{l}\Large$R_{Q_f}(f) =0.25$ \end{tabular}}}}%
  \end{picture}%
\endgroup%

    }

    \caption{
    \emph{Left panel:} Initial situation, the ML model classifies individual
    with starting features $x_0$ either negatively (in blue) or positively (in
    red). Its risk is denoted by $R_P(f)$. Points classified negatively are given the
    opportunity to move to the decision boundary (yellow dotted arrows).
    \emph{Right panel:}
    The points close enough to the boundary accept recourse and move towards the decision boundary.
    The risk with recourse, $R_{Q_f}(f)$, is then higher, because at the
    decision boundary the uncertainty about the true class is maximal,
    and the points that accepted recourse are now more likely to be
    misclassified.}
    \label{fig:intro_illustration}
    \vspace{-0.4cm}
\end{figure}
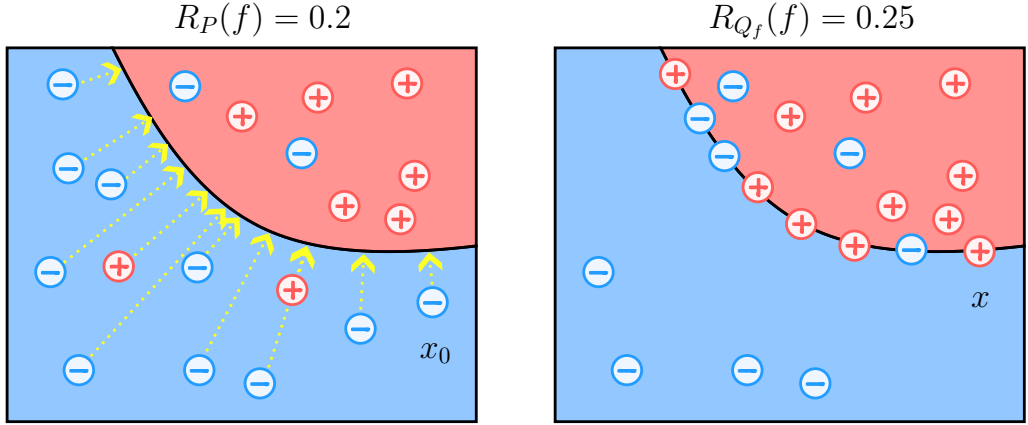
\paragraph{Main Contributions}

In this work, we study the effects of recourse on the classification
accuracy at the population level. All our results are obtained in the
context of a new learning-theoretic framework, which we introduce in
Section~\ref{sec:setting}. Accuracy is measured by the risk, which is
the expected loss of a given classifier. When recourse is provided, it
changes the distribution of the data, and hence the risk. We are
primarily interested in whether recourse makes the risk go up or down.
To answer this question it matters how the class probabilities of the
users change upon receiving recourse. We distinguish between the
\emph{compliant} case, in which these class probabilities truly improve,
and the \emph{defiant} case, in which the class probabilities do not
improve at all (for instance because the users are trying to game the
system). In Section~\ref{sec:bayes_risk} we show that, if the classifier
is optimal without recourse, then recourse will be harmful, because it
increases the risk both for the compliant and for the defiant case. The
reason is that recourse pushes users towards the decision boundary,
where the class uncertainty is higher, which therefore leads to more
mistakes. See Figure~\ref{fig:intro_illustration} for an illustration.
Section~\ref{sec:probabilist_classification} extends these results to
probabilistic classifiers that are only near-optimal, which allows for
estimation errors, and to surrogate losses like, e.g., the cross-entropy
loss. In Section~\ref{sec:strategic} we recognize that the party
deploying the classifier may strategically choose their classifier in
order to minimize the resulting risk after providing recourse. We obtain
separate results for the defiant and compliant cases, which show that
there is an incentive to preemptively undo the effect of recourse. For
the defiant case, this makes the risks with and without recourse
identical, so implementing recourse only places a burden on all parties,
without any resulting advantage. For the compliant case, the risk with
recourse does decrease, so this is the only case where we do observe an
advantage to providing recourse. Finally, in
Section~\ref{sec:experiments} we corroborate our theoretical results by
experiments, in which we observe the risk increase for a large majority
of the experiments both on synthetic data and on real data. We also
provide the code that produced the results of our experiments as a 
GitHub repository.\footnote{\href{
    https://github.com/HiddeFok/consequences-of-recourse}{
github.com/HiddeFok/consequences-of-recourse
}}

\paragraph{Not Reliably Beneficial}

In summary, our findings show that there are many common cases in which
recourse is harmful, because it leads to worse classification accuracy.
This suggests that instead of debating how to provide recourse, we
should rethink whether the current approach to recourse is desirable at
all. Notably, there is no escape by pointing to exceptions in which
recourse is beneficial, e.g., our results on strategic
classification for the compliant case, or by pointing to specific
examples where it is beneficial in practice: if recourse is not reliably
beneficial nearly all the time, then it is not suitable to be broadly
adopted.

\subsection{Further Related Work}\label{sec:related_work}
\paragraph{Causality and Algorithmic Recourse}

The difference between the defiant and the compliant case has already
been noted in the causal algorithmic recourse community. This has lead
to counterfactual methods with guarantees for actual improvement of the
class probabilities, which ensure that we are closer to the compliant
case \cite{konig2021causal,konig2023improvement}. However, our results
show that, even in the fully compliant case, recourse may still be
harmful. Such unexpected harmful effects of well-intended interventions have
also been found in the context of fairness \citep{delayed18hardt}. More
generally, it has been pointed out that users can only act on counterfactual
recommendations if these take the causal relation between the user's actions
and their features into account \citep{karimi2021algorithmic}.
Our framework is general enough to express such causal interventions,
because they only affect the risk via their effect on the distribution
of the data.

\paragraph{Strategic Classification}

Strategic classification considers the effect of deploying a classifier
in an environment with strategic players, who want to change their
features in order to influence how they are classified
\citep{hardt2016strategic,
levanon2021strategic,hardt2020causal, tsirtsis2019optimal,
yaton2021encourage}. This makes the distribution of the data dependent on $f$
as well, because the behavior of the players depends on the
classifier $f$. The more abstract setting in which there can be any dependence
between $f$ and the data distribution, has been studied under the heading
of performative prediction \citep{performative2020hardt,
mofakhami2023performative}. Our results about strategizing in
Section~\ref{sec:strategic} are a special case of strategic
classification, in which the behavior of the players is guided by the
recourse mechanism. In contrast to previous results that mostly
considered how to minimize the risk in $f$ while taking the dependence
of $f$ on the distribution into account, our aim is to quantify the
difference in the risk when we compare the settings with and without
recourse.


\section{Framework and Main Definitions}\label{sec:setting}

In this section we formalize the effect of recourse by comparing the
risk in the situation without recourse to the risk with recourse
applied.

\subsection{General Framework}

We consider binary classification, in which users with corresponding
features $x$ from a closed, convex domain $\domain \subseteq \reals^d$ will be
classified into classes $\classes = \{-1,+1\}$. We assume a model $f :
\domain \to \predset$ has already been trained. This may be a
deterministic classifier, with $\predset = \{-1,+1\}$, or a
probabilistic classifier, with $\predset = [0,1]$, for which $f(x)$
represents the probability that $x$ should be classified as $+1$. The
error of a prediction $\hat y \in \predset$ with respect to the true
label $y \in \classes$ is measured by a loss function $\loss : \predset
\times \classes \to \reals$. For instance, for deterministic predictions
$\hat y \in \{-1,+1\}$, the $0/1$ loss is $\loss(\hat y, y) = \ind\{\hat
y \neq y\}$, and, for probabilistic predictions $\hat y \in [0,1]$, the
log loss or cross-entropy loss is $\loss(\hat y,y) = \tfrac{1}{2}(1 + y)
\ln \frac{1}{\hat y} + \tfrac{1}{2}(1-y) \ln \frac{1}{1-\hat y}$.

In the absence of recourse, the data will consist of pairs $(X_0,Y)$ from
$\domain \times \classes$ with distribution $P$, and the quality of~$f$
is evaluated by its risk
\[
  R_P(f) = \E_{(X_0,Y) \sim P}[\loss(f(X_0),Y)].
  \tag{Risk without Recourse}
\]
A classifier $f_P^* \in \argmin_f R_P(f)$, which minimizes the risk, is
called Bayes-optimal. For instance, for $0/1$ loss, $f_P^*(x_0) =
\sign{P(Y=1|X_0=x_0) - \tfrac{1}{2}}$ is Bayes-optimal. Throughout the
paper, we take the sign function $\sign{z}$ to be $+1$ if $z \geq 0$ and
$-1$ for $z < 0$.

When we add recourse to the mix, a user first arrives with feature
vector $X_0$, which is drawn according to the marginal distribution of
$P$ on $\domain$. Then, depending on the original features $X_0$, the
specifics of the recourse protocol, and the model $f$, the user's
features are transformed into new features $X \in \domain$. 
Here, $X$ may
be a deterministic function of $X_0$, but in general it can also depend
on $X_0$ in a non-deterministic way if the recourse protocol is
randomized or when the user's response to recourse is not fully
predictable. Finally, a label $Y$ is generated, and we let $Q_f$ denote
the resulting distribution of $(X_0,X,Y)$. The resulting risk is then
measured under the marginal distribution of $(X,Y)$ under~$Q_f$:
\[
  R_{Q_f}(f) = \E_{(X,Y) \sim Q_f}[\loss(f(X),Y)].
  \tag{Risk with Recourse}
\]
Thus, the marginal distribution of $X_0$ under $Q_f$ is always the same
as under~$P$. Note further that $f$ influences the risk with recourse in
two ways: directly via its predictions $f(X)$ and indirectly via its
effect on the distribution $Q_f$. Except for Section~\ref{sec:strategic}
we will think of $f$ as fixed, and we will simplify notation by writing
$Q$ instead of $Q_f$. 

As motivated in the introduction, we care about the accuracy of
classifiers at the population level. This is measured by the risk, so 
we will say that recourse is beneficial if the risk under $Q$ is smaller
than the risk under $P$, and harmful otherwise.


\subsection{Specializing the Framework}
\label{sec:specializing}

The framework above is so general that it can represent any mechanism for
providing recourse.  In order to say something concrete, we have to specialize
it further.

\paragraph{Effect on the Label Distribution}

Naively, we might expect that changing the user's features from $X_0$ to
$X$ would also change their label distribution from $P(Y|X_0)$ to
$P(Y|X)$, but what actually happens depends on the underlying causal
effect of providing recourse
\citep{hardt2020causal,konig2023improvement},
and in general any effect on the label distribution is possible. We will
focus on two extreme cases which differ in whether individuals fully comply
with or fully defy this naive expectation:
\begin{itemize}[labelwidth=\widthof{Compliant}, leftmargin=2cm]
    \item[\textbf{(Compliant)}]\label{case:compl}
        $\displaystyle Q(Y\mid  X_0, X) = P(Y \mid  X)$. 
        The change in features causes a true change in label
        probability.
    \item[\textbf{(Defiant)}]\label{case:def}
        $\displaystyle Q(Y\mid  X_0, X) = P(Y \mid  X_0)$. 
        The user only changes their features, without altering their
        label probability.
\end{itemize}
We state all our results in terms of these two extreme cases. However, 
those results can easily be generalized to an intermediate setting
by taking a convex combination of the compliant and defiant measure. 
The convex combination would then carry over towards the theoretical
results. 

The defiant case has also been referred to as ``gaming''\footnote{We avoid this
terminology in the context of algorithmic recourse, because users may follow a
recourse recommendation in good faith and still not change their label
probability.} \citep{konig2023improvement,performative2020hardt}. It is
illustrated well by the following example by \citet{konig2023improvement}:
consider a classifier which classifies whether a patient is infected with Covid
based on their symptoms. Then, taking cough drops to suppress coughing may
change the classification without changing the true probability of being
infected. This behaviour could also appear when there is no recourse
considered. In that case, we assume that it is already modelled in the
distribution $P$. In our setting, the act of giving recourse is what gives the
users the opportunity to ``game'' the system, when they were not before.  It
should not come as a complete surprise that our results show that the risk
increases when giving recourse in the defiant case. The more surprising
conclusions are that even in the compliant case it is possible to observe a
risk increase and that strategizing against the risk increase in the defiant
case comes with its own negative consequences. 

\paragraph{Recourse Mechanism}

We will think of class $+1$ as being favorable to the users, while
class $-1$ is undesirable to them. For instance, $+1$ might represent a
bank loan being granted, while $-1$ means that the loan application is
rejected. Whenever a user with features $X_0$ is classified as $f(X_0) = -1$
by a deterministic classifier, they may request recourse. Many prominent
approaches
\citep{wachter2017counterfactual,ustun2019actionable,karimi2020model,pawelczyk2022exploring}
to algorithmic recourse provide the user with a counterfactual
explanation $\Xcf_0 = \varphi(X_0)$ which is the solution to an
optimization problem of the form
\begin{equation}\label{eq:cf_optimization}
    \Xcf_0 \in \argmin_{z \in \domain \colon f(z) = +1}
        c(X_0, z),
\end{equation}
where $c(x_0, z)$ models the cost for the user of moving from $x_0$ to
$z$. This can describe many different cost mechanisms, and can even be
used to express constraints like monotonicity in an Age feature or
consistency with a causal model, by assigning large cost to any $z$ that
violates the constraints. For the optimization problem in
\eqref{eq:cf_optimization} to be well-defined, we need to assume that
the set
\begin{equation}\label{eqn:closed_plusdomain}
  \{x \in \domain \mid f(x) = +1\}
  \qquad
  \text{is closed.}
\end{equation}
A consequence of this, is that a point on the decision boundary of a
classifier will be classified as class~$+1$. So, in order for $f_P^*$ to
satisfy this condition for $0/1$ loss, it matters that we defined
$\sign{0} = +1$ above. For many of our results, we will further assume
that for any points $x_0$ and $x$ the cost
\begin{equation}\label{eqn:monotonic_cost}
    \begin{split}
        z \mapsto c(x_0, z) \quad 
        \text{increases monotonically on}\\
        \text{the line segment from $x_0$ to $x$},
    \end{split}
\end{equation}
which means that larger changes require more effort from the user. Under
this assumption, $\varphi$ always maps users $x_0$ in the negative class
to the decision boundary; for users in the positive class, recourse does
not do anything and $\varphi(x_0) = x_0$. (See
Lemma~\ref{lemma:boundary} in Appendix~\ref{app:proofs_setting}.) If the
user implements the counterfactual explanation exactly, then $X=\Xcf_0$,
but they might also deviate from it in a stochastic way, which would
make $X$ a noisy approximation of $\Xcf_0$
\citep{pawelczyk2022probabilistically}. For simplicity, we will focus on
the noiseless case with $X=\Xcf_0$. We do explicitly take into account the
fact that not all users might receive recourse and that each user
has a choice in whether to implement it. Let $B \in \{0,1\}$ be an
indicator variable for whether recourse is received and implemented,
with conditional probability $\Pr(B=1\mid X_0) = r(X_0)$. It
then follows that
\[
  X
    = (1-B) X_0 + B \Xcf_0
    = (1-B) X_0 + B \varphi(X_0)
    \, .
\]
Note that, when $f(X_0) = +1$, we always have $X = X_0$ irrespective of
$B$, because $\varphi(X_0) = X_0$ as mentioned above.  
Some examples of possible $r$ functions are:
\begin{itemize}
    \item $r(x_0) = 1$. All users implement recourse;
    \item 
        $r(x_0) = \ind\{\| x_0- \varphi(x_0)\| \le D\}$ for some
        $D > 0$. Only those users within
        distance $D$ of the decision boundary implement recourse;
    \item 
        $r(x_0) = e^{-\frac{\|x_0 - \varphi(x_0)\|^2}{2\sigma^2}}$ 
        for some $\sigma^2 > 0$. All users implement recourse
        with some probability and that probability  
        is exponentially decreasing in the squared distance they have
        to cover, with a bandwidth $\sigma^2$.
\end{itemize}

\paragraph{More Complex Counterfactuals}
In our setting, we define a counterfactual explanation to be a solution to an
optimization problem \eqref{eq:cf_optimization}. In accordance with the early
counterfactual methods developed in \citep{wachter2017counterfactual,
karimi2020model, laugel2017inverse}. For the Compliant case, some of the
recent counterfactual methods \citep{laugel2019issues, kanamori2020dace,
parmentier2021optimal} can have more complex cost functions, with the goal of
generating more realistic, feasible and robust counterfactuals. These methods
produce counterfactuals that do not lie on the decision boundary, but may lie
further into the positive class. For these methods, the cost function still has
a component that depends on the Euclidean distance. The other component of the
cost function depends on something called the outlier score, which indeed does
not satisfy our requirement in Equation~\eqref{eqn:monotonic_cost}. However,
depending on how these terms are balanced, the counterfactual point will still
be close to the decision boundary, so the setting can still be approximated
by the cases that we study. 

Moreover, these methods are not designed to pay attention to the class
probability $P(Y=+1|X)$, which is crucial to circumvent our result, and may
therefore still produce counterfactuals for which $P(Y=+1|X) \approx
\frac{1}{2}$.  For instance, the methods in \citep{kanamori2020dace,
parmentier2021optimal} do pay attention to the marginal distribution of~$X$,
but not to $P(Y=+1|X)$.  Consequently, all current robustness metrics or
outlier metrics that do not increase the risk after providing recourse will do
so by luck, but not by design. 

We also like to remark that the results of the Defiant case will hold, whatever
the cost function may be. This is because the conditional probability of being
in the positive class does not change in that setting. 

\section{Risk Increase for the Bayes-Optimal Classifier}
\label{sec:bayes_risk}

In this section we present our first main result, which relates the
risk with recourse under $Q$ to the risk without recourse under $P$. The
result implies that the risk with recourse is larger, because recourse
will move data from a region where the prediction is relatively certain,
for example $P(Y=-1 | X_0) = 0.9$, to the decision boundary, where
things are the least certain, because $P(Y=+1 | X) = 1/2$. We also
illustrate this in an example with Gaussian data. The proofs and
additional details for the example in this section can be found in
Appendix~\ref{app:proofs_bayes_risk}. 

\begin{restatable}[Bayes-Optimal Classifier Risk Increase]{theorem}{ExactIncrease}
    \label{thm:exact_risk_increase}
    Let~$\ell$ be the $0/1$ loss, and assume the setting of
    Section~\ref{sec:specializing} (i.e., \eqref{eq:cf_optimization},
    \eqref{eqn:closed_plusdomain}, \eqref{eqn:monotonic_cost}).
    Suppose that $P(Y=1 |X_0 =x) =
    \tfrac{1}{2}$ for all $x$ on the decision boundary of $f^{*}_P$.
    Then
    \begin{enumerate}[label=(\alph*)]
        \item For the defiant case, 
        \begin{align}\label{eq:exact_bayes_def}
            R_{Q}(f_P^{*}) 
            & = P(B=1, f_P^{*}(X_0) = -1, Y = -1) 
                - P(B=1, f_P^{*}(X_0) = -1, Y = +1) + R_P(f_P^{*}) \nonumber\\
            &\ge R_P(f_P^{*});
        \end{align}
        \item For the compliant case,
        \begin{align}\label{eq:exact_bayes_comp}
            R_{Q}(f_P^{*}) 
            &= \tfrac{1}{2}P(B=1, f_P^{*}(X_0) = -1) 
             - P(B=1, f_P^{*}(X_0)=-1, Y=1) + R_P(f_P^{*})\nonumber\\ 
            &\ge R_P(f_P^{*}).
        \end{align}
    \end{enumerate}
    Both inequalities are strict if $P(B=1, f_P^{*}(X_0) = -1) > 0$, i.e., 
    if the probability of recourse in the negative class is non-zero.
\end{restatable}
Theorem~\ref{thm:exact_risk_increase} gives an explicit expression for
the risk with recourse when $f_P^*$ is the Bayes classifier for~$P$.
Under very general conditions, it shows that providing recourse always
increases the risk, for any recourse probability function $r$ and any
monotonically increasing cost function $c$!

\subsection{Gaussian Example}\label{sec:gaussian_example_bayes}
\begin{figure}
    \begin{center}
    \includegraphics[scale=0.95]{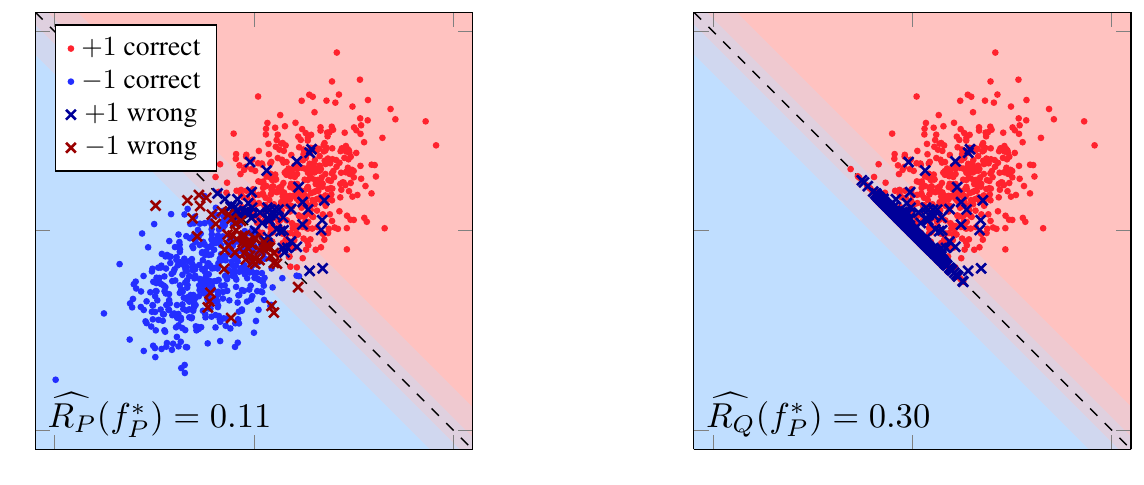}
    \end{center}
    \caption{\emph{Left:} Bayes classifier, original
    predictions; \emph{Right:} predictions after providing
    recourse in the compliant case.}
    \label{fig:gaussian_example}
\end{figure}
We proceed with a simple example that can be analyzed in closed form and
plotted visually. We assume the data is generated as follows. Let $P(X_0 \mid
Y=y)$ be $\normal(\mu, \Sigma)$ for $y=+1$ and $\normal(\nu, \Sigma)$ for
$y=-1$ for positive definite $\Sigma$, with equal prior class probabilities
$P(Y=-1) = P(Y=+1) = \tfrac{1}{2}$. For simplicity, we will assume that
$\|\mu\|_{\Sigma^{-1}} = \|\nu\|_{\Sigma^{-1}}$, 
where $\|\mu\|_{\Sigma^{-1}}^2 = \left<\mu,\Sigma^{-1}\mu \right>$
and set $\theta \coloneqq
\Sigma^{-1}(\mu -\nu) $. Then, the optimal classifier is known to be
$f_P^{*}(x_0) = \sign{x_0^{\top}\theta}$, and the Bayes risk can be expressed
in terms of the distribution function $\Phi$ of a standard normal distribution:
$R_P(f_P^{*}) = \Phi(-\tfrac{1}{2}\|\mu - \nu\|_{\Sigma^{-1}})$.  For Euclidean
cost $c(x_0,z) = \|x_0 - z\|$, providing 
recourse boils down to projecting onto the hyperplane $\{x \in \mathcal{X}
\mid x^{\top}\theta = 0\}$ and this projection can be expressed analytically by
a linear transformation $\varphi(x_0) = \left(I - \frac{\theta
\theta^{\top}}{\|\theta\|^2}\right)x_0$.

We see the effect of providing recourse on the data distribution and the risk
for the compliant case in Figure~\ref{fig:gaussian_example}. We have taken $\mu
= (+1,+1)^\top$, $\nu = (-1,-1)^\top$ and $\Sigma = \begin{psmallmatrix} 1 &
    0.5\\ 0.5 & 1 \end{psmallmatrix}$, and set $r(x_0) = 1$. In this case,
    $R_P(f_P^{*}) = \Phi(-\tfrac{1}{2}\|\mu - \nu\|_{\Sigma^{-1}}) \approx
    0.1$. The figure also shows empirically that the risk increases, which
    matches the prediction by Theorem~\ref{thm:exact_risk_increase} that
$R_Q(f_P^{*}) = \tfrac{1}{4} + \tfrac{1}{2}\Phi(-\tfrac{1}{2}\|\mu -
\nu\|_{\Sigma^{-1}}) \approx 0.31$.
The defiant case is not shown, because it would result in a similar picture,
but with $R_Q(f_P^{*}) = \frac{1}{2}$.


\section{Risk Increase for Probabilistic Classifiers}
\label{sec:probabilist_classification}

In practice, we do not have direct access to the Bayes-optimal
classifier and the classifier is learned from data.  In this section, we
therefore drop the requirement that the classifier is exactly
Bayes-optimal. We will further consider probabilistic classifiers
$\ProbModel : \domain \to [0,1]$. Thresholding $\ProbModel$ then leads
to a binary classifier $\BinModel(x) = \sign{\ProbModel(x)
-\tfrac{1}{2}}$. 
We will compare the risk with recourse to the risk without recourse,
first for the $0/1$ loss and then for a class of surrogate losses that
includes the cross-entropy loss. The assumptions we make differ, but in
both cases the conclusion is that the risk with recourse exceeds the
risk without recourse when~$g$ is sufficiently accurate. The proofs for
this section are presented in Appendix~\ref{app:proofs_prob_class}.


\subsection{Risk Increase for the $0/1$ loss}
\label{sec:risk-increase-01}

We again focus on the $0/1$ loss first. We can
handle the defiant case without further assumptions. But for the
compliant case we require that $g$ is highly accurate in the sense that
its decision boundary is close to Bayes-optimal. A simple sufficient
requirement would be that there exists $\epsilon \geq 0$ such that
\begin{align}\label{hyp:supremum}\tag{A}
  \abs{\tfrac{1}{2} - P(Y=1  \mid X_0=x)} \leq
  \epsilon\\
  \qquad \text{for all $x$ such that $\ProbModel(x) = 1/2$.} \notag
\end{align}
This gives a uniform control over deviations anywhere along the decision
boundary of~$g$. At the cost of a slightly more complicated condition,
this uniform bound can be relaxed to an average under the distribution
over the points from the negative class that get mapped to the decision
boundary of~$g$:
\begin{equation}\label{en:calibration_assumption}\tag{B}
  \int\limits_{\{x_0 : \ProbModel(x_0) <  \nicefrac{1}{2}\} }
    |\tfrac{1}{2} - P(Y=1  \mid X = \varphi(x_0))| P(\text{d}x_0)
  \leq \epsilon.
\end{equation}
Assuming $g$ is continuous, it will equal $g(x) = 1/2$ for all points
$x$ on its decision boundary. When~$\varphi$ maps all points~$x_0$ from
the negative class to the decision boundary of $g$, it follows that
\eqref{hyp:supremum} implies~\eqref{en:calibration_assumption}.

We are now ready to derive an analogous result to 
Theorem~\ref{thm:exact_risk_increase}: 
\begin{restatable}[Probabilistic Classifier Risk Increase, $0/1$ loss]{theorem}{BoundRisk}
    \label{thm:bound_risk}
    Let $\ell$ be the $0/1$ loss. Let $\ProbModel:\mathcal{X} \to [0,1]$
    be a continuous, probabilistic classifier, and define $\BinModel(x) =
    \sign{\ProbModel(x) -\tfrac{1}{2}}$. Assume
    \eqref{eq:cf_optimization}, \eqref{eqn:closed_plusdomain},
    \eqref{eqn:monotonic_cost} 
    from Section~\ref{sec:specializing}. Then, 
    \begin{enumerate}[label=(\alph*)]
        \item For the defiant case, 
            \begin{align}
            \label{eq:defiant_prob_exact}
                R_{Q}(f) 
                &= P(B=1, f(X_0) = -1, Y=-1)  
                -  P(B=1, f(X_0) = -1, Y=+1)
                + R_P(f)
            .\end{align}
        Moreover, $R_Q(\BinModel) \geq R_P(\BinModel)$ if and only if
        \begin{equation}\label{eq:defiant_prob_increase}
            P(Y=-1  \mid B=1, \BinModel(X_0) = -1)
              \geq \tfrac{1}{2}.
        \end{equation}
    \end{enumerate}
    If we additionally assume that $\ProbModel$ satisfies
    \eqref{en:calibration_assumption} with $0\le \epsilon \le \tfrac{1}{2}$, then
    \begin{enumerate}[resume*]
        \item For the compliant case,  $R_Q(\BinModel)$ 
            is lower and upper bounded by
            \begin{align}
            \label{eq:compliant_prob_exact}
                (\tfrac{1}{2} &\pm \epsilon)P(B = 1, \BinModel(X_0) =-1) 
                + P(\BinModel(X_0) = +1, Y=-1)
                + P(B=0, \BinModel(X_0) = -1, Y=1 )
            .\end{align}
            Moreover, $R_Q(\BinModel) \geq R_P(\BinModel)$ if
            \begin{equation}\label{eq:cond-risk-increase-01}
              P(Y=-1 \mid B=1, \BinModel(X_0) = -1) \geq \tfrac{1}{2} + \epsilon.
            \end{equation}
    \end{enumerate}

\end{restatable}
Equations~\eqref{eq:defiant_prob_increase} and
\eqref{eq:cond-risk-increase-01} express that the class $-1$ is actually
more likely (with a margin of $\epsilon$) conditional on the set of
points in the negative class that accept recourse. This will be satisfied 
when $f$ is a reasonably accurate classifier. The intuition is that in
this case moving points to the decision boundary is harmful, because
they are more likely to be misclassified there.
We also note that, for $\epsilon = 0$, $f$ will be equal to the
Bayes-optimal classifier, and the condition is always satisfied, so we
recover the conclusion from Theorem~\ref{thm:exact_risk_increase} that
the risk will always increase. 


\subsection{Risk Increase for Surrogate Losses}
\label{sec:risk-increase-surrogate}

In this section, we investigate the scenario in which the loss is not
the $0/1$ loss, but rather a surrogate loss. We are primarily thinking
of the cross-entropy loss, as defined in Section~\ref{sec:setting}, but
our result also applies to any other loss for probabilistic predictions
$\hat y \in [0,1]$ which is such that $\loss(1/2,-1) = \loss(1/2,+1)$ is
constant.

\begin{restatable}[Probabilistic Classifier Risk Increase, Surrogate Loss]{theorem}{RiskIncreaseSurrogate}
\label{th:risk-increase-surrogate}
    Let $\ell : [0,1] \times \{-1,+1\} \to \reals$ be any loss such that
    $\loss(1/2,-1) = \loss(1/2,+1)= c$ for some constant $c$. Let
    $\ProbModel:\mathcal{X} \to [0,1]$ be a continuous, probabilistic
    classifier, and define $\BinModel(x) = \sign{\ProbModel(x)
    -\tfrac{1}{2}}$. Further assume \eqref{eq:cf_optimization},
    \eqref{eqn:closed_plusdomain}, \eqref{eqn:monotonic_cost} 
    from Section~\ref{sec:specializing}. Then,
    both for the defiant and for the compliant case, we have $R_Q(g)
    \geq R_P(g)$ if and only if
    \begin{equation}\label{eqn:good_surrogate_classifier}
      \E_P\big[\loss(g(X_0),Y) \mid f(X_0) = -1, B=1\big] \leq c.
    \end{equation}
\end{restatable}
Condition~\ref{eqn:good_surrogate_classifier} means that, on average
over users from the negative class who receive recourse, the loss should
be lower than the value of the loss at the decision boundary. This means
that $g$ should be a reasonably accurate classifier, which performs
better on this group than simply predicting $1/2$. But it is much weaker
than requiring that $g$ should be close to Bayes-optimal, as we did in
Theorems~\ref{thm:exact_risk_increase} and~\ref{thm:bound_risk}. We can
get away with this weaker requirement, because, at the decision
boundary, $g(x) = 1/2$ and therefore the loss is $c$ regardless of the
underlying distribution of $Y$. This is also the reason that the defiant
and the compliant case coincide.

\section{Strategic Classification}
\label{sec:strategic}

So far we have assumed that the classifier $\BinModel$ was fixed, but when the
party deploying $\BinModel$ knows in advance that they will need to provide
recourse, they have an incentive to strategically choose $\BinModel$ in order to
minimize the resulting risk under~$Q$. In this section, we study the
result of strategizing for both the defiant and compliant scenario. 
Before presenting our results, we first introduce the
part of the setup that is common to both. At the end of the section, we
reflect on our findings in a short discussion.

\subsection{Common Setup}

Throughout this section we focus on binary classifiers 
$\BinModel : \domain \to \{-1,+1\}$ 
with the $0/1$ loss. And, since $\BinModel$ is now variable, we write
$Q_{\BinModel}$, $\varphi_\BinModel$ and $r_\BinModel$ 
instead of $Q$, $\varphi$ and~$r$. We
assume that anyone either accepts or rejects recourse
deterministically, i.e., that $r_\BinModel(x_0) \in \{0, 1\} $ for all $x_0$. 
And we also assume that the classifier $\BinModel$ is selected from a
restricted class of functions $\modelclass$. For notational sake, define 
$\varphi_\BinModel^{r}(x_0) =r_\BinModel(x_0)\varphi_{\BinModel}(x_0) + (1-r_\BinModel(x_0))x_0 $.
Under the effect of recourse, $\modelclass$ transforms into
\[
    \modelclass^{r}_{\varphi}
    \defeq \{
        x_0 \mapsto 
        \BinModel(\varphi_\BinModel^{r}(x_0)) 
        \mid \BinModel \in \modelclass
    \}
\, .
\]
We say that $\modelclass$ is \emph{invariant under recourse} if, for any
$\BinModel \in \modelclass$, there exists a unique $\BinModel' \in
\modelclass$ such that $\BinModel'$ with recourse is equivalent to $f$
without recourse, i.e.,  $\BinModel'(\varphi_\BinModel^{r}(x_0)) =
\BinModel(x_0)$ for all $x_0$. This implies, in particular, that 
$\modelclass^{r}_\varphi = \modelclass$.
As a concrete example, one can think of linear classification, with
recourse defined as bringing any point within distance less than $D>0$
of the decision boundary to the positive class. 
In this example, shifting the original classifier by $D$ orthogonally to
the decision boundary in the direction of the positive class gives another equivalent classifier: it is thus invariant under recourse. 
Details for this example and another one are provided in Appendix~\ref{app:protecting}. 

\subsection{Defiant Case}
In the defiant case, the setting above implies that providing recourse
does not change the risk:
\begin{restatable}[Strategizing in the Defiant Case]{theorem}{StrategizingDefiant}
\label{th:strategizing-defiant}
  Let $\ell$ be the $0/1$ loss, assume \eqref{eq:cf_optimization},
  \eqref{eqn:closed_plusdomain}, \eqref{eqn:monotonic_cost} from
  Section~\ref{sec:specializing} with $r(x_0) \in \{0, 1\} $ 
  for all $x_0 \in \mathcal{X}$, and
  suppose $\modelclass$ is invariant under recourse. 
  Then, providing
  recourse in the defiant case does not change the risk when the party deploying the
  classifier strategizes to minimize their risk over $\modelclass$:
  \[
    \min_{\BinModel \in \modelclass} R_{Q_\BinModel}(\BinModel)
      = \min_{\BinModel \in \modelclass} R_{P}(\BinModel)
      \, .
  \]
\end{restatable}
Intuitively, the reason is that in the defiant case it is strategically
optimal to maintain the original decision boundary, because users do not
really change upon receiving recourse. This is possible when~$\modelclass$ 
is recourse invariant, because then there is always a
function available that compensates for the effect of recourse. Recourse
therefore has no effect on the final decisions, but instead only places
a burden on users who have to implement it and on the party deploying
the classifier, which has to provide a recourse mechanism. In this case,
recourse therefore has only negative effects, and may be considered
harmful. 
We prove Theorem~\ref{th:strategizing-defiant} in Appendix~\ref{app:protecting}. 


\subsection{Compliant Case}

In the compliant case, the situation is different and strategizing can
actually improve the risk. We require the following definition. 
\begin{definition}\label{def:compliant_move}
  Suppose $\modelclass$ is recourse
  invariant, and let $ \BinModel \in \argmin_{\BinModel \in \modelclass} R_P(\BinModel)$ be a
  minimizer of the risk without recourse. Let $ \BinModel' \in
  \modelclass$ be the (unique) 
  classifier such that $ \BinModel'(\varphi_{ \BinModel'}(x_0)) =
   \BinModel(x_0)$ for all $x_0 \in \mathcal{X}$ and define 
  $\Delta$ to be, 
  \[
    \Delta \defeq 
      \E_{(X_0, Y) \sim P}[\loss( \BinModel(X_0), Y)]
      - \E_{(X_0, Y) \sim Q_{ f'}}[\loss( \BinModel(X_0), Y)]
      \, .
  \]
\end{definition}
Here, the function $ \BinModel'$ compensates the effect of giving recourse 
for the original classifier $ \BinModel$, and it exists by recourse invariance.
The quantity $\Delta$ measures the change in risk when we fix the classifier to be $ \BinModel$,
but the data are either generated by $P$ (no recourse) or $Q_{ \BinModel'}$ 
(recourse for the classifier $ \BinModel'$). 
Intuitively, $\Delta$ measures the effect of recourse on the distribution of
users when the strategy is to choose a function $ \BinModel'$ that
compensates for the effect of recourse. We generally expect recourse to
move users further into the positive class, and therefore to make it
more certain that their class label will indeed be $Y = +1$, which means
that~$\Delta$ would be positive. 
A detailed example is provided in Appendix~\ref{app:protecting}. 

\begin{restatable}[Strategizing in the Compliant Case]{theorem}{StrategizingCompliant}
\label{th:strategizing-compliant}
  Let $\ell$ be the $0/1$ loss, assume \eqref{eq:cf_optimization},
  \eqref{eqn:closed_plusdomain}, \eqref{eqn:monotonic_cost} from
  Section~\ref{sec:specializing} 
  and suppose $\modelclass$ is invariant under recourse. Let $\Delta$ be as
  in Definition~\ref{def:compliant_move}. 
  Then, the risk after providing recourse in the compliant case 
  can be bounded in terms of the risk without recourse when the party deploying the
  classifier strategizes to minimize their risk over $\modelclass$,
  \[
    \min_{f \in \modelclass} R_{Q_f}(f)
      \leq R_{Q_{ f'}}( f')
      = \min_{f \in \modelclass} R_{P}(f) - \Delta,
  \]
  where $ f'$ is as in Definition~\ref{def:compliant_move}.
\end{restatable}

When $\Delta$ is positive, this shows that providing recourse will be
beneficial. In Appendix~\ref{app:protecting}, we
prove Theorem~\ref{th:strategizing-compliant} and expand the example of
Section~\ref{sec:gaussian_example_bayes} by showing that $\Delta > 0$ in that
case.


\subsection{Discussion}

We observe that both in the defiant and in the compliant case, an
appealing strategy for the party deploying the classifier is to
compensate for the effect of recourse by changing their classifier in a
way that maintains the original decision boundary. This implies that all
users get classified exactly the same way as without recourse, and the
only effect of recourse is to change the conditional distribution of
$Y$. For instance, in a bank loan setting, the same customers would get
the loan, but some customers might be required to reduce their
probability of defaulting before getting it.

%
\begin{figure*}[b]
    \centering
    \includegraphics[scale=0.95]{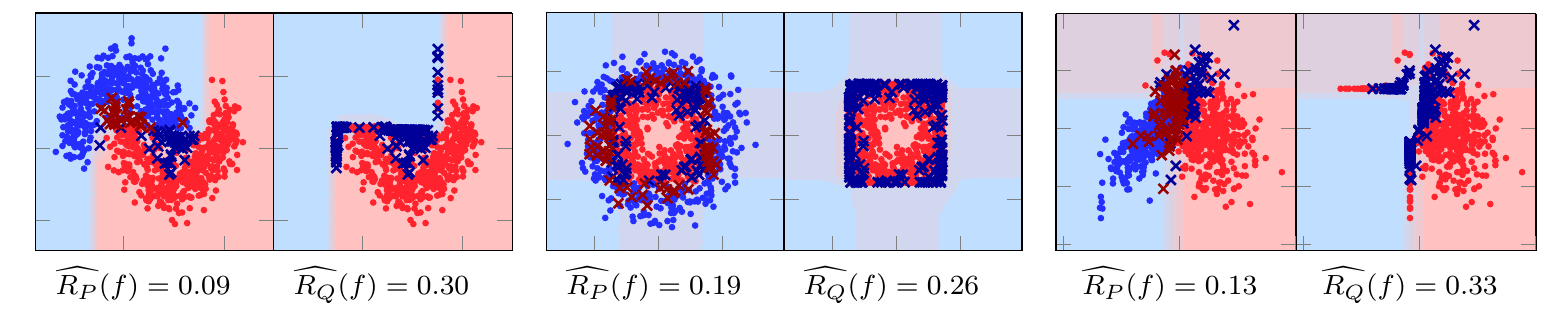}
    \caption{From left to right: Moons, Circles and Gaussian datasets.
    The left image for each shows the classifications with gradient
    boosted trees; the right image shows the effect of giving recourse.}
    \label{fig:synth_data}
\end{figure*}

\section{Experiments}\label{sec:experiments}

In addition to our theoretical results, we perform several experiments that
showcase the possible increase in risk by providing recourse. We conduct these
on synthetic data and real data. In both cases we generate $Y$ according to the
compliant setting, and except if it is stated otherwise, recourse is provided
for all $x_0$ that are classified as class~$-1$. Further details for all
the experiments are available in Appendix~\ref{app:experiments}.

\subsection{Synthetic Data}

\begin{table}[t]
    \caption{Estimated risks on synthetic data sets. Lower
    risk is bold.}
    \label{tbl:synthetic_results}
    \centering 
    \scriptsize
    \setlength\tabcolsep{4pt}
    \begin{tabular}{m{4cm}cccccc}
        \toprule
        & 
        \multicolumn{2}{c}{Moons data} & 
        \multicolumn{2}{c}{Circles data} & 
        \multicolumn{2}{c}{Gaussians data} \\
        & 
        $R_P$ & $R_Q$ & $R_P$ & $R_Q$ & $R_P$ & $R_Q$ \\
        \toprule
        Logistic Regression (LR) 
         & \textbf{0.13} $\pm$ \textbf{0.01} & 0.32 $\pm$ 0.03
         & 0.52 $\pm$ 0.02 & \textbf{0.36} $\pm$ \textbf{0.02}
         & \textbf{0.14} $\pm$ \textbf{0.01} & 0.35 $\pm$ 0.05
        \\

        GradientBoostedTrees (GBT) 
         & \textbf{0.07} $\pm$ \textbf{0.03} & 0.25 $\pm$ 0.09
         & \textbf{0.18} $\pm$ \textbf{0.02} & 0.27 $\pm$ 0.03
         & \textbf{0.14} $\pm$ \textbf{0.02} & 0.35 $\pm$ 0.06
        \\

        Decision Tree (DT) 
         & \textbf{0.09} $\pm$ \textbf{0.02} & 0.28 $\pm$ 0.08
         & \textbf{0.19} $\pm$ \textbf{0.02} & 0.26 $\pm$ 0.04
         & \textbf{0.14} $\pm$ \textbf{0.02} & 0.39 $\pm$ 0.08
        \\

        Naive Bayes (NB) 
         & \textbf{0.13} $\pm$ \textbf{0.02} & 0.33 $\pm$ 0.03
         & \textbf{0.16} $\pm$ \textbf{0.02} & \textbf{0.16} $\pm$ \textbf{0.03}
         & \textbf{0.16} $\pm$ \textbf{0.03} & 0.29 $\pm$ 0.03
        \\

        QuadraticDiscriminantAnalysis (QDA) 
         & \textbf{0.13} $\pm$ \textbf{0.01} & 0.33 $\pm$ 0.04
         & \textbf{0.16} $\pm$ \textbf{0.01} & \textbf{0.16} $\pm$ \textbf{0.03}
         & \textbf{0.13} $\pm$ \textbf{0.02} & 0.38 $\pm$ 0.05
        \\

        Neural Network(4) (NN 1)
         & \textbf{0.12} $\pm$ \textbf{0.04} & 0.27 $\pm$ 0.09
         & \textbf{0.26} $\pm$ \textbf{0.22} & \textbf{0.27} $\pm$ \textbf{0.09}
         & \textbf{0.14} $\pm$ \textbf{0.02} & 0.37 $\pm$ 0.04
        \\

        Neural Network(4, 4) (NN 2)
         & \textbf{0.07} $\pm$ \textbf{0.07} & 0.25 $\pm$ 0.06
         & \textbf{0.18} $\pm$ \textbf{0.01} & \textbf{0.23} $\pm$ \textbf{0.04}
         & \textbf{0.12} $\pm$ \textbf{0.00} & 0.40 $\pm$ 0.00
        \\

        Neural Network(8) (NN 3)
         & \textbf{0.07} $\pm$ \textbf{0.06} & 0.22 $\pm$ 0.04
         & \textbf{0.17} $\pm$ \textbf{0.02} & \textbf{0.20} $\pm$ \textbf{0.02}
         & \textbf{0.13} $\pm$ \textbf{0.02} & 0.36 $\pm$ 0.04
        \\

        Neural Network(8, 16) (NN 4)
         & \textbf{0.03} $\pm$ \textbf{0.01} & 0.26 $\pm$ 0.03
         & \textbf{0.17} $\pm$ \textbf{0.01} & \textbf{0.19} $\pm$ \textbf{0.04}
         & \textbf{0.13} $\pm$ \textbf{0.02} & 0.39 $\pm$ 0.04
        \\

        Neural Netowrk(8, 16, 8) (NN 5)
         & \textbf{0.03} $\pm$ \textbf{0.01} & 0.26 $\pm$ 0.03
         & \textbf{0.17} $\pm$ \textbf{0.01} & \textbf{0.19} $\pm$ \textbf{0.04}
         & \textbf{0.13} $\pm$ \textbf{0.02} & 0.39 $\pm$ 0.04
        \\

    \bottomrule
    \end{tabular}
\end{table}

The synthetic data consist of the $3$ datasets shown in
Figure~\ref{fig:synth_data}, all in $2$ dimensions: a Moons dataset,
which consists of two translated
semi-circles with Gaussian noise; a Circles dataset, which consists of
two nested circles with Gaussian noise; and a final dataset consisting
of $2$ Gaussians with different means and covariances. Counterfactuals
for $c(x_0,z) = \|z-x_0\|$ were computed by a brute force search to
find the closest point $z$ with $f(z) = +1$ from a dense grid over
$\domain$.

A summary of the estimated risks for a variety of classifiers can be seen in
Table~\ref{tbl:synthetic_results}. The experiments were also repeated~$10$
times to estimate confidence bounds for the risks. We also performed
experiments where not everyone accepts the counterfactual. We distinguish
between two cases. In the first, everyone has the same probability $r(x_0) = p$
of accepting the counterfactual. In the second case, the probability of
accepting is determined by the distance towards the counterfactual explanation.
We choose to model this probability as $r(x_0)=e^{-\tfrac{1}{2\sigma^2} \|x_0 -
\varphi(x_0)\|^2}$.  These results are summarised in the plots in
Figure~\ref{fig:risk_p_sigma_increase} and show that giving more points
recourse, generally increases the risk.  The plots show the risk difference
$R_Q - R_P$ on the y-axis, and either~$p$ or $\sigma$ on the x-axis. We see a
clear linear dependence on~$p$ in the first case, which is predicted by our
results. See Appendix~\ref{app:linear_relation} for a derivation of this fact. 

Looking at Table~\ref{tbl:synthetic_results}, We observe that the risk increases
in all cases for the Moons and Gaussians dataset.  With the Circles dataset,
most of the risks with recourse had a higher mean, but the confidence bounds
were overlapping. The biggest exception was logistic regression on the Circles
dataset. Here, the risk decrease happens because logistic regression has a
linear decision boundary, which is severely inappropriate for this data.
Without recourse, almost half of the class $+1$ is misclassified, because the
linear boundary cuts both circles. If the points of the outer circle, which are
of class $-1$, are projected onto this line, a large portion will land inside
the inner circle, where the conditional probability of class $+1$ will be
significantly larger than $\tfrac{1}{2}$.
\begin{figure}[t]
    \centering
    \includegraphics[scale=0.62]{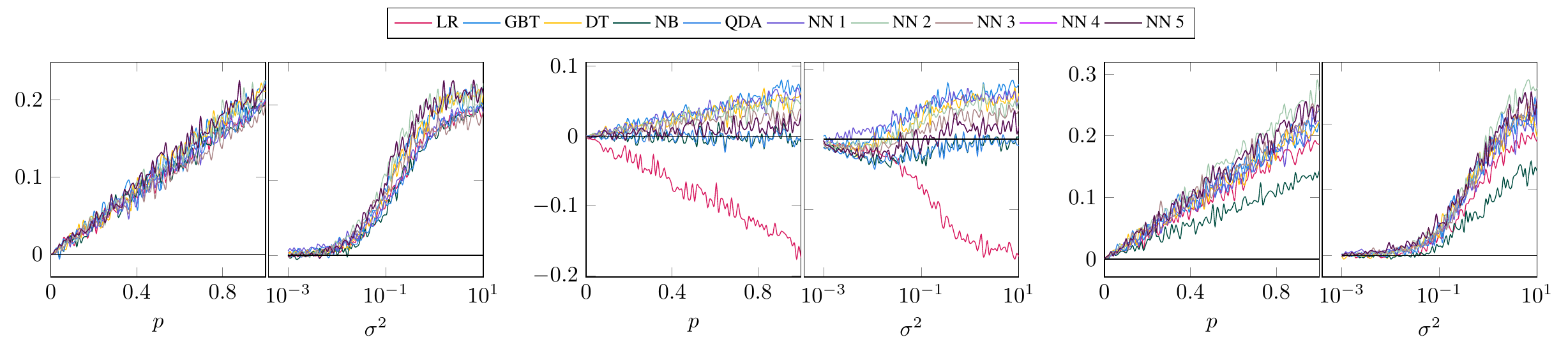}    
    \caption{From left to right: Moons, Circles and Gaussian datasets.
        The left image for each shows the risk difference when $p \in [0,1]$; 
        the right image shows the risk difference when 
        $\sigma^2 \in [10^{-3}, 10^1]$ on a logarithmic scale}
    \label{fig:risk_p_sigma_increase}
\end{figure}
\subsection{Real Data}
For the real datasets we use the \textit{Give me Credit}, \textit{Census
Income}, and \textit{Home Equity Line of Credit (HELOC)} datasets, from the \verb|CARLA| Python package
\citep{pawelczyk2021carla}. All features were normalized to $[0, 1]$. We
compare various classifiers, and $3$ counterfactual methods: Wachter's method
\citep{wachter2017counterfactual}, the Growing Spheres method
\citep{laugel2017inverse}, and Counterfactual Genetic Search (CoGS)
\citep{virgolin2023cogs}. The main challenge on real data is that we do not have access to the true
conditional distribution, $P(Y  \mid X)$. This distribution is needed to
sample $Y$ after obtaining $X = \varphi(X_0)$ through
recourse. To circumvent this issue we reserve a large portion of the
data to train a calibrated classifier for the conditional
probabilities. 

A summary of the estimated risks can be seen in
Table~\ref{tbl:real_data_adult}, which is shown here, and
Tables~\ref{tbl:real_data_credit} and \ref{tbl:real_data_heloc} in
Appendix~\ref{app:experiments}. Every experiment was repeated $10$
times to estimate confidence bounds. For the Census and HELOC datasets,
providing recourse indeed increases the risk in most cases, as predicted
by our theoretical results, but in a small number of cases the risk goes
down. The exceptions may be explained if the accuracy of the classifiers
in class $-1$ is not significantly better than random guessing, as
required by \eqref{eq:cond-risk-increase-01}. The third dataset is the
Credit data: in this case the classes are very unbalanced (class $-1$
occurs only $7\%$ of the time), and, because we did not fully finetune
the classifiers, several of them perform worse than the $7\%$ base error
rate that can be obtained by always predicting class $+1$. Since our
theoretical results only apply to high accuracy classifiers (see again
Equation~\ref{eq:cond-risk-increase-01}), this leaves room for the
possibility that providing recourse might actually decrease the risk in
this case. The estimates in Table~\ref{tbl:real_data_credit} indeed
suggest that this is happening, but, unfortunately, because of the class
imbalance, the variance in our risk estimates is too large to
definitively confirm that this indeed occurs.
\begin{table}[b]
    \caption{Estimated risks on the Census dataset.
    Lower risk in bold.}
    \label{tbl:real_data_adult}
    \centering 
    \footnotesize
    \setlength\tabcolsep{4pt}
    \begin{tabular}{m{1cm}|cccccc}
        \toprule
        &
        \multicolumn{2}{c}{Wachter} & 
        \multicolumn{2}{c}{GS} & 
        \multicolumn{2}{c}{CoGS}\\
        \hline
        & 
        $R_P$ & $R_Q$ & $R_P$ & $R_Q$ & $R_P$ & $R_Q$\\
        \toprule
        \bottomrule
LR & \textbf{0.21} $\pm$ \textbf{0.03}& 0.30 $\pm$ 0.02
& \textbf{0.22} $\pm$ \textbf{0.02}& 0.33 $\pm$ 0.03
& \textbf{0.21} $\pm$ \textbf{0.02}& 0.35 $\pm$ 0.03
\\

GBT & 0.15 $\pm$ 0.01& \textbf{0.05} $\pm$ \textbf{0.01}
& \textbf{0.15} $\pm$ \textbf{0.02}& \textbf{0.17} $\pm$ \textbf{0.12}
& \textbf{0.16} $\pm$ \textbf{0.02}& \textbf{0.35} $\pm$ \textbf{0.25}
\\

DT & \textbf{0.25} $\pm$ \textbf{0.04}& \textbf{0.23} $\pm$ \textbf{0.05}
& \textbf{0.24} $\pm$ \textbf{0.04}& \textbf{0.46} $\pm$ \textbf{0.19}
& \textbf{0.24} $\pm$ \textbf{0.06}& 0.46 $\pm$ 0.11
\\

NB & \textbf{0.18} $\pm$ \textbf{0.02}& 0.76 $\pm$ 0.02
& \textbf{0.18} $\pm$ \textbf{0.02}& 0.77 $\pm$ 0.03
& \textbf{0.18} $\pm$ \textbf{0.03}& 0.81 $\pm$ 0.03
\\

QDA & \textbf{0.21} $\pm$ \textbf{0.02}& 0.76 $\pm$ 0.03
& \textbf{0.20} $\pm$ \textbf{0.02}& 0.75 $\pm$ 0.04
& \textbf{0.20} $\pm$ \textbf{0.02}& 0.81 $\pm$ 0.03
\\

NN 1 & \textbf{0.16} $\pm$ \textbf{0.01}& \textbf{0.27} $\pm$ \textbf{0.11}
& \textbf{0.16} $\pm$ \textbf{0.02}& 0.25 $\pm$ 0.06
& \textbf{0.16} $\pm$ \textbf{0.02}& 0.28 $\pm$ 0.06
\\

NN 2 & \textbf{0.16} $\pm$ \textbf{0.02}& 0.32 $\pm$ 0.05
& \textbf{0.15} $\pm$ \textbf{0.02}& 0.31 $\pm$ 0.04
& \textbf{0.15} $\pm$ \textbf{0.02}& 0.35 $\pm$ 0.06
\\

NN 3 & \textbf{0.16} $\pm$ \textbf{0.02}& 0.36 $\pm$ 0.07
& \textbf{0.16} $\pm$ \textbf{0.02}& 0.30 $\pm$ 0.06
& \textbf{0.16} $\pm$ \textbf{0.02}& 0.33 $\pm$ 0.06
\\

NN 4 & \textbf{0.16} $\pm$ \textbf{0.01}& 0.38 $\pm$ 0.06
& \textbf{0.16} $\pm$ \textbf{0.02}& 0.34 $\pm$ 0.06
& \textbf{0.15} $\pm$ \textbf{0.02}& 0.38 $\pm$ 0.08
\\

NN 5 & \textbf{0.16} $\pm$ \textbf{0.01}& 0.38 $\pm$ 0.06
& \textbf{0.16} $\pm$ \textbf{0.02}& 0.34 $\pm$ 0.05
& \textbf{0.15} $\pm$ \textbf{0.02}& 0.38 $\pm$ 0.08
\\
    \end{tabular}
\end{table}
%

\section{Conclusion}\label{sec:conclusion}

We demonstrated, analytically and empirically, that in many cases the risk
will increase when recourse is provided. This implies that recourse can
be harmful at the population level, and therefore for a large group of
users. In such cases, alternative types of explanations might be called
for. One interesting alternative direction is the existing work on
contestability, which addresses the question of whether an algorithmic
decision is correct according to common sense, moral or legal
standards \citep{freiesleben2022intriguing}.
As a possibility for future work, our framework might be extended by
also accounting for the cost incurred by the users when implementing
recourse, e.g.,  by adding a scaled version of $c(X_0,X)$ to $R_Q(f)$.
Assuming positive costs, this would make recourse even less appealing,
and lead to the conclusion that it is harmful in an even larger number
of cases. Another extension, which would be more interesting to explore,
would be to apply our framework in cases where the users and the party
deploying the classifier have different loss functions. Then the
relation between $f$ and Bayes-optimal decisions for the user's loss
would be broken, which might lead to different conclusions.

\acknowledgments{D. Garreau acknowledges support from the French government, through the NIM-ML project (ANR-21-CE23-0005-
	01). 
This work was realized while he was working at Universit\'e C\^ote d'Azur (France). 
T. Van Erven was supported by the
    Netherlands Organization for Scientific Research (NWO) under grant number
    VI.Vidi.192.095.}

\bibliography{corbib}


\appendix


\section{Proofs for Section~\ref{sec:setting}}
\label{app:proofs_setting}

Recall that, for any $x_0\in \domain$, we choose
\[
\varphi(x_0) \in \argmin_{z \in \domain \colon f(z) = +1} c(x_0,z)
\, .
\]

\begin{lemma}
\label{lemma:boundary}
Assume that $\{z \in \mathcal{X} \mid f(z) = +1\}$ is a closed subset of $\R^d$, 
and that the cost $c$ satisfies Assumption~\eqref{eqn:monotonic_cost}.  
Then $\varphi(x_0)$ belongs to the boundary of $\{z \in \mathcal{X}  \mid f(z) = +1\}$ 
for all $x_0 \in \mathcal{X}$ such that $f(x_0) = -1$
and $\varphi(x_0) =x_0$ if $x_0$ belongs to $\{z \in \mathcal{X} \mid 
f(z)= + 1\}$.
\end{lemma}

\begin{proof}
First, we remark that $c(x_0, z)$ is minimized by $z=x_0$, 
whenever $x_0 \in \{z \in \mathcal{X} \mid  f(z)= + 1\}$, which shows that 
$\varphi(x_0) = x_0$ for all $x_0$ in $\{z \in \mathcal{X} \mid f(z)= + 1\}$.

Moving towards the case of $x_0$ such that $f(x_0) = -1$. Let us set
$x_1=\varphi(x_0)$ and $A=\{z \in \mathcal{X} \mid f(z) = +1\}$.  By
contradiction, assume that $x_1$ does not belong to the boundary of $A$.  Since
by construction $x_1\in A$, then we must have $x_1\in A^{\mathrm{o}}$, the
interior of $A$. Consider the point $x_2$ that lies on the line segment $[x_0, x]$ 
and on the boundary of $A$. As $c(x_0, z)$ is increasing for $z \in [x_0, x]$ 
it must be that
\[
c(x_0, x_2) <  c(x_0, x)
.\] 
This contradicts the definition of $x_1$ and concludes the proof. 
\end{proof}

Note, that the requirement on $c$ can be further weakened by assuming
that there is a path between $x_0$ and $x$ such that $z \mapsto c(x_0, z)$ is
increasing along this path. So, we do not necessarily need straight line segments. 
The proof of this statement is analogous to the proof of Lemma~\ref{lemma:boundary}.


\section{Further Details for Section~\ref{sec:bayes_risk}}\label{app:proofs_bayes_risk}
Before we start the proof of the main result in Section~\ref{sec:bayes_risk},
we will introduce notation and some additional results. The conditional 
distributions of $Y$ given $X_0$ will be denoted by
\begin{align*}
    p_{+}(x) 
    &\coloneqq P(Y=+1 \mid X_0=x) = 1 - P(Y=-1 \mid X_0=x ) 
    \eqqcolon 1 - p_{-}(x).
\end{align*}
Now, we will prove a general result about expressing the risk under $Q$, 
of which Theorem~\ref{thm:exact_risk_increase} is a consequence. 
Every expectation $\mathbb{E}$ will be with respect to $P$ in this section.
\begin{lemma}\label{lem:express_Q_risk}
    Let $\ell$ be a loss function with 
    $\ell(y, y) = 0$, and assume the setting of
    Section~\ref{sec:specializing} (i.e., \eqref{eq:cf_optimization},
    \eqref{eqn:closed_plusdomain}, \eqref{eqn:monotonic_cost}).
    Suppose that $P(Y=1 |X_0 =x) =
    \tfrac{1}{2}$ for all $x$ on the decision boundary of $f$.
    Then 
    \begin{enumerate}[label=(\alph*)]
        \item For the defiant case, 
            \begin{align}\label{eq:gen_bayes_def}
                R_{Q}(\BinModel) 
                    &= \ell(1, -1)P(Y=-1)\mathbb{E}\left[ r(X_0) | Y=-1 \right] 
                    + \mathbb{E}\left[ (1 - r(X_0))\ell(\BinModel(X_0),
                    Y) \right];
            \end{align}
        \item For the compliant case, 
            \begin{align}\label{eq:gen_bayes_comp}
                R_{Q}(\BinModel) 
                &= \ell(1, -1)\mathbb{E}\left[ r(X_0)p_{-}(\varphi(X_0))\right] 
                + \mathbb{E}\left[ (1 - r(X_0))\ell(\BinModel(X_0), Y) \right] 
             .\end{align}
    \end{enumerate}
\end{lemma}
\begin{proof}
Before distinguishing between the $2$ cases, we expand the expression for the risk under $Q$ as
\begin{align*}
    R_Q&(\BinModel) 
    = \int\limits_{\mathcal{X}^2 \times \mathcal{Y}}\ell(\BinModel(x), y) 
        Q(\text{d}y  \mid x, x_0) Q(\text{d}x | x_0) P(\text{d}x_0) \\
    &= \int\limits_{\mathcal{X} \times \mathcal{Y}}r(x_0)\ell(\BinModel(\varphi(x_0)), y) 
        Q(\text{d}y  \mid \varphi(x_0), x_0) P(\text{d}x_0)
    + \int\limits_{\mathcal{X} \times \mathcal{Y}}
    (1 - r(x_0))\ell(\BinModel(x_0), y) 
        P(\text{d}y  \mid x_0) P(\text{d}x_0) \\
    &= \ell(1, -1)\int\limits_{\mathcal{X}}r(x_0) 
    Q(Y=-1 \mid \varphi(x_0), x_0) P(\text{d}x_0) \tag{since $f(\varphi(x_0)) = 1$}\\
    &\phantom{\quad} + \mathbb{E}\left[ (1 - r(X_0))\ell(\BinModel(X_0), Y) \right]
.\end{align*}
We focus now on the integral in the first term.

\textbf{Defiant Case:} The first term in the expression in the above display
becomes
\begin{align*}
    \int\limits_{\mathcal{X}}r(x_0)
        Q(Y=-1  \mid \varphi(x_0), x_0) P(\text{d}x_0)
    &= \int\limits_{\mathcal{X}}r(x_0) 
        P(Y=-1  \mid  x_0) P(\text{d}x_0)\\
    &= P(Y=-1)\int\limits_{\mathcal{X}}r(x_0) 
        P(\text{d}x_0  \mid Y=-1)\\
    &= P(Y=-1)\mathbb{E}\left[ r(X_0)  \mid Y=-1 \right] 
.\end{align*}

\textbf{Compliant Case:} The first term now becomes
\begin{align*}
    \int\limits_{\mathcal{X}}r(x_0)
        Q(\text{d}y  \mid \varphi(x_0), x_0) P(\text{d}x_0)
    &= \int\limits_{\mathcal{X} }r(x_0)
        P(\text{d}y  \mid  \varphi(x_0)) P(\text{d}x_0) \\
    &= \int\limits_{\mathcal{X}}r(x_0)
        P(Y=-1  \mid  \varphi(x_0)) P(\text{d}x_0) \\
    &= \int\limits_{\mathcal{X}}r(x_0) 
        p_{-}(\varphi(x_0))P(\text{d}x_0) \\
    &= \mathbb{E}\left[ r(X_0) p_{-}(\varphi(X_0)) \right] 
.\end{align*}
\end{proof}
We are now ready to prove Theorem~\ref{thm:exact_risk_increase}.
\ExactIncrease*
\begin{proof}
For both cases, we will first prove the equality and then show that
the expectation is always non-negative for the inequality. From those
proofs it can be seen how the strict inequality is derived.
We apply Lemma~\ref{lem:express_Q_risk} to both cases. 
Remark that $\ell(1, -1) = \ell(-1, 1) = 1$ and rewrite the 
common term as 
\begin{align}
    \mathbb{E}\left[ (1 - r(X_0))\ell(f_P^{*}(X_0), Y) \right]
    &\phantom{\quad}=  R_P(f_P^{*}) 
    - \E\left[ r(X_0)\ind\left(f_P^{*}(X_0) \neq Y \right) \right] \notag\\
    &\phantom{\quad}= R_P(f_P^{*}) 
    - P(B=1, f_P^{*}(X_0) \neq Y)\label{eq:compl_r}.
\end{align}

\textbf{Defiant Case:} In this case, we rewrite the first term in 
\eqref{eq:gen_bayes_def} to get
\begin{equation}\label{eq:cond_r}
    P(Y=-1)\mathbb{E}\left[ r(X_0) | Y = -1\right] 
    = \E\left[ r(X_0) \ind(Y=-1) \right].
\end{equation}
Combining expressions \eqref{eq:compl_r} and \eqref{eq:cond_r} gives 
the result, 
\begin{align}
    R_Q(f_P^{*}) 
    &= \E\left[ r(X_0)( \ind(Y=-1) - \ind(f_P^{*}(X_0) \neq Y ))\right]  
    + R_P(f_P^{*}) \label{eq:def_bayes_exp} \notag \\
    &= P(B=1, Y=-1) -  P(B=1, f_P^{*}(X_0) \neq Y) 
    + R_P(f_P^{*})\\
    & = P(B=1, f_P^{*}(X_0) = -1, Y = -1)
    - P(B=1, f_P^{*}(X_0) = -1, Y = +1) \notag\\
    & \quad + R_P(f_P^{*})\notag
.\end{align}
It remains to show that the difference of the first two probabilities
is positive. We return to the formulation in terms of expectations
and indicator functions. We can rewrite the indicator 
functions giving rise to those probabilities as 
\begin{align*}
    \ind(&f_P^{*}(x_0) = -1, y=-1) -  \ind(f_P^{*}(x_0) =-1, y=1)
    = \ind(f_P^{*}(x_0) = -1)(\ind(y=-1) -  \ind(y=1))
.\end{align*}
The expectation in \eqref{eq:def_bayes_exp} now becomes
\begin{align}\label{eq:def_integral}
    \int\limits_{\mathcal{X}\times \mathcal{Y}} r(x_0) 
    &( \ind (y=-1) 
    - \ind (f_P^{*}(x_0) \neq y) ) 
        P(\text{d}x_0,\text{d}y) \notag\\
    &=\int\limits_{\{f_P^{*} = -1\} \times \mathcal{Y}} 
        r(x_0) ( \ind (y=-1) 
        - \ind(y=1) ) P(\text{d}y  \mid X=x) 
        P(\text{d}x_0)\notag\\
    &= \int\limits_{\{f_P^{*} =-1\} }
        r(x_0)(p_{-}(x_0) - p_{+}(x_0))
        P(\text{d}x_0)
.\end{align}
Now, by $f_P^{*}$ being the Bayes optimal classifier we know that
$p_{-}(x_0)\ge p_{+}(x_0)$ for all $x_0$ on 
$\{f_P^{*} = -1\} $. So, we see that the integral in 
\eqref{eq:def_integral} is non-negative.

\textbf{Compliant Case:} We note that 
$p_{+}(\varphi(x_0)) = p_{-}(\varphi(x_0)) = \tfrac{1}{2}$ for any 
$x_0$ with $f_{P}^{*}(x_0) = -1$, because those points 
are projected onto the decision boundary by assumption~\eqref{eqn:monotonic_cost}
and Lemma~\ref{lemma:boundary}. The points on the decision boundary 
of the Bayes classifier are exactly where the probability of being
either class is $\tfrac{1}{2}$, by assumption. The first expectation in 
\eqref{eq:gen_bayes_comp} can now be written as
\begin{align}
        \mathbb{E}[r(X_0)p_{-}(\varphi(X_0))] 
        &= \tfrac{1}{2}\mathbb{E}[r(X_0)\ind(f_P^{*}(X_0) = -1)]
        + \mathbb{E}[r(X_0)p_{-}(X_0)\ind(f_P^{*}(X_0) = +1)] \notag\\
        &= \tfrac{1}{2}P(B=1, f_P^{*}(X_0) = -1) 
        + P(B=1, f_P^{*}(X_0) = + 1, Y=-1).\label{eq:db_r}
\end{align}
We note that the second probability in \eqref{eq:db_r} cancels the second 
probability \eqref{eq:compl_r} partly. First we write the latter
probability as
\begin{align*}
    P(B&=1, f_P^{*}(X_0) \neq Y) 
    =P(B=1, f_P^{*}(X_0) =1, Y=-1) 
    + P(B=1, f_P^{*}(X_0) =-1, Y=1)
.\end{align*}
Subtracting both probabilities gives
\begin{align*}
    P(B=1, &f_P^{*}(X_0) = 1, Y=-1)
    - 
    P(B=1, f_P^{*}(X_0) \neq Y)
    = -P(B=1, f_P^{*}(X_0) =-1, Y=1)
.\end{align*}
Substituting \eqref{eq:db_r} and \eqref{eq:compl_r} into the expression for 
$R_Q(f_P^{*})$ gives
\begin{align*}
    R_Q(f_{P}^{*}) 
    &= \mathbb{E}\left[ r(X_0)p_{-}(\varphi(X_0))\right] 
    + \mathbb{E}\left[ (1 - r(X_0))\ell(\BinModel(X_0), Y) \right] \\
    &= \mathbb{E}\left[ r(X_0)p_{-}(\varphi(X_0))\right] 
        + R_P(f_P^{*}) 
        - P(B=1, f_P^{*}(X_0) \neq Y) \\
    &= \tfrac{1}{2}P(B=1, f_P^{*}(X_0) = -1) 
        - P(B=1, f_P^{*}(X_0)=-1, Y=1)
        + R_P(f_P^{*})
.\end{align*}
To derive the necessary inequality, we focus again on the first two probabilities 
and write this explicitly as an integral. This integral is given by
\begin{align*}
    \tfrac{1}{2}P(B=1, 
    &f_P^{*}(X_0) = -1) 
    - 
    P(B=1, f_P^{*}(X_0)=-1, Y=1)\\
    &= \int\limits_{\{f_P^{*} =-1\}}
        r(x_0)(\tfrac{1}{2} - p_{+}(x_0))
        P(\text{d}x_0) \ge 0
\end{align*}
Where we have used that the on the set 
$\{x_0 \in \mathcal{X}  \mid f_P^{*}(x_0)=-1\}$ it must be that
$p_{+}(x_0)  \le \tfrac{1}{2}$, because  $f_P^{*} $ is the
Bayes classifier.

The strict inequality follows by remarking that the difference of the integrand in 
both integrals of the defiant and compliant case will be strictly positive on 
some positive probability set, if $P(B=1, f_P^{*}=-1)> 0$. 
\end{proof}

\subsection{Additional Details Gaussian Example in 
Section~\ref{sec:gaussian_example_bayes}}
\label{sec:additional_gaus_details}
In Section~\ref{sec:gaussian_example_bayes} it is claimed that the Bayes risk 
can be written as 
$R_P(f_P^{*}) = \Phi(-\tfrac{1}{2} \|\mu -\nu\|_{\Sigma^{-1}})$. Here, we
show this and additionally derive the Bayes optimal classifier for general
$\mu, \nu \in \mathbb{R}^{d}$.

The conditional distribution can be calculated explicitly. 
Let 
\begin{align*}
p_\mu(x) = e^{-\frac{1}{2}(x - \mu)^{\top}\Sigma^{-1}(x-\mu)}
\quad\text{ and }\quad
p_\nu(x) = e^{-\frac{1}{2}(x - \nu)^{\top}\Sigma^{-1}(x-\nu)},
\end{align*}
then
\begin{align*}
    P(Y=1  \mid X_0 = x) 
    &= \frac{
        P(Y=1)p_\mu(x)
    }{
        P(Y=1) p_\mu(x)
        +
        P(Y=-1) p_\nu(x)
    }\\
    &= \frac{
         1
    }{
        1 + e^{-x^{\top}\Sigma^{-1}(\mu - \nu) 
            + \frac{1}{2}(\|\mu\|_{\Sigma^{-1}}^2 - \|\nu\|_{\Sigma^{-1}}^2)}
    }
.\end{align*}
From this we see that for $\theta = \Sigma^{-1}(\mu - \nu)$ and
$\theta_0 = -\frac{1}{2}(\|\mu\|_{\Sigma^{-1}}^2 - \|\nu\|_{\Sigma^{-1}}^2)$ 
the Bayes classifier is given by 
$f_P^{*}(x) = \sign{x^{\top}\theta + \theta_0} $. We can now
calculate the Bayes risk by first rewriting this risk as
\begin{align}
    R_P(f_P^{*}) 
    &= \tfrac{1}{2}P(f_P^{*}(X_0) =-1 \mid Y =1) 
    + \tfrac{1}{2}P(f_P^{*}(X_0) =1 \mid Y =-1) \notag\\
    &= \tfrac{1}{2}P(X_0^{\top}\theta + \theta_0 < 0 \mid Y =1)
    + \tfrac{1}{2}P(X_0^{\top}\theta + \theta_0 \ge 0\mid Y =-1)
        \label{eqn:bayes_risk_gaussians}
.\end{align}
As $X_0$ is Gaussian, conditional on $Y$, we know that 
$X_0^{\top}\theta + \theta_0$ is also Gaussian. For $Y=1$, 
we get 
$\normal(\mu^{\top}\theta + \theta_0, \|\theta\|_{\Sigma^{-1}}^2)$ and 
for $Y=-1$ we get 
$\normal(\nu^{\top}\theta + \theta_0, \|\theta\|_{\Sigma^{-1}}^2)$.
Translating and rescaling allows us to rewrite the probabilities in 
\eqref{eqn:bayes_risk_gaussians} in terms of the CDF $\Phi$ of 
the standard normal distribution, 
\begin{align*}
    P(X_0^{\top}\theta 
    + \theta_0 < 0 \mid Y =1) 
    &= \Phi\left( 
    \frac{
        -\mu^{\top}\theta - \theta_0
    }{
        \|\theta\|_{\Sigma^{-1}}
    } \right) \\
    &=\Phi\left(
    \frac{
        -\|\mu\|_{\Sigma^{-1}}^2 + \mu^{\top}\Sigma^{-1}\nu 
        + \frac{1}{2}(\|\mu\|_{\Sigma^{-1}}^2 - \|\nu\|_{\Sigma^{-1}}^2)
    }{
        \|\mu -\nu\|_{\Sigma^{-1}}
    }\right) \\
    &=\Phi\left(
    \frac{
        -\frac{1}{2}\|\mu - \nu\|_{\Sigma^{-1}}^2
    }{
        \|\mu -\nu\|_{\Sigma^{-1}}
    }\right) 
    = \Phi( -\tfrac{1}{2}\|\mu - \nu\|_{\Sigma^{-1}} ) 
.\end{align*}
Analogously, we would get
\begin{align*}
    P(X_0^{\top}\theta + \theta_0 \ge  0 \mid Y =-1) 
    = \Phi( -\tfrac{1}{2}\|\mu - \nu\|_{\Sigma^{-1}} )
.\end{align*}
Combining the two probabilities gives the desired result.


\section{Proofs of Section~\ref{sec:probabilist_classification}}
\label{app:proofs_prob_class}
In this section we present all the previous and additional results of 
Section~\ref{sec:probabilist_classification}.
\subsection{Proof of Theorem~\ref{thm:bound_risk}}
\BoundRisk*
\begin{proof}
    \textbf{Defiant Case:} We again use Lemma~\ref{lem:express_Q_risk} which gives us
    \begin{align*}
        R_Q(\BinModel) 
        &= P(Y=-1)\E[r(X_0)  \mid Y=-1] 
            + \E[(1-r(X_0)\ind(f(X_0)\neq Y)]\\
        &= P(Y=-1, B=1) + P(f(X_0) \neq Y) 
            - P(f(X_0) \neq Y, B=1) \\
        &= P(Y=-1, B=1) - P(f(X_0) \neq Y, B=1) 
            + R_P(f) \\
        &= P(B=1, f(X_0) = -1, Y=-1) 
            - P(B=1, f(X_0) = -1, Y=+1) 
            + R_P(f)
    .\end{align*}
    To derive the second claim, we upper bound $R_P(\BinModel)$ by
    $R_Q(f)$. We see that the $R_P(f)$ term drops on both sides and we are left with
    \begin{align*}
        P(B=1, &f(X_0) = -1, Y=-1) 
        \le 
        P(B=1, f(X_0) = -1, Y=+1) 
    \end{align*}
    Conditioning on $\{B=1, f(X_0) = -1\} $ and cancelling the common 
    terms gives us
    \begin{align*}
        P(Y=+1  \mid f(X_0) = -1, B=1) 
        &\le P(Y=-1 \mid f(X_0) = -1, B=1), 
        & &\iff\\
        P(Y=-1  \mid f(X_0) = -1, B=1) 
        & \ge \tfrac{1}{2}
    .\end{align*}
    \textbf{Compliant Case:} We apply Lemma~\ref{lem:express_Q_risk}. Note, that 
    Assumption~\ref{en:calibration_assumption} and Lemma~\ref{lemma:boundary} 
    tell us that on the set $\{x_0 \in \mathcal{X}  \mid \BinModel(x_0)= -1\} $ 
    we have that $\tfrac{1}{2} - \epsilon < p_{-}(\varphi(X_0)) \le \tfrac{1}{2} + \epsilon$ in 
    expectation. 
    For the first expectation we get the upper bound
    \begin{align*}
        \mathbb{E}\left[ r(X_0) p_{-}(\varphi(X_0)) \right] 
        &= \mathbb{E}\left[r(X_0)\ind \{\BinModel(X_0) = -1\} p_{-}(\varphi(X_0))\right] 
        + \mathbb{E}\left[r(X_0) \ind \{\BinModel(X_0) = 1\} p_{-}(X_0)\right]\\
        &\le  (\tfrac{1}{2} + \epsilon) P(\BinModel(X_0) = -1, B=1)
        + P(\BinModel(X_0) =1, Y=-1, B=1)
    .\end{align*}
    Analogously, for the lower bound we get
    \begin{align*}
        \mathbb{E}\left[ r(X_0) p_{-}(\varphi(X_0)) \right]  
        &\ge  (\tfrac{1}{2} - \epsilon) P(\BinModel(X_0) = -1, B=1)
        + P(\BinModel(X_0) =1, Y=-1, B=1)
    .\end{align*}
    We write the second expectation as follows in this case, 
    \begin{align*}
        \mathbb{E}\left[ (1 - r(X_0)) \ell(\BinModel(X_0), Y) \right] 
        = P(f(X_0) \neq Y, B=0)
    .\end{align*}
    This leaves us with
    \begin{align}
        R_Q(\BinModel) 
        &\le (\tfrac{1}{2} + \epsilon)P(\BinModel(X_0) = -1, B=1) 
        + P(\BinModel(X_0) = 1, Y=-1, B=1) 
        + P(B=0, f(X_0) \neq Y)\notag\\
        &\le (\tfrac{1}{2} + \epsilon)P(\BinModel(X_0) = -1, B=1) 
        + P(\BinModel(X_0) = 1, Y=-1) 
        + P(f(X_0) =-1, Y=1, B=0)\label{eq:Q_risk_upper_bound}
    .\end{align}
    Similarly, for the lower bound we get
    \begin{align}
        R_Q(\BinModel) 
        &\geq(\tfrac{1}{2} - \epsilon) P(\BinModel(X_0) = -1, B=1) 
        + P(\BinModel(X_0) = 1, Y=-1) 
        + P(f(X_0) =-1, Y=1, B=0)\label{eq:Q_risk_lower_bound}
    \end{align}
   Combining expressions \eqref{eq:Q_risk_upper_bound} 
   and \eqref{eq:Q_risk_lower_bound} gives
   the desired lower and upper bound.

   We move to the second claim. This time, we upper bound $R_P(\BinModel)$ 
   by the derived lower bound. This gives us
   \begin{align*}
    P(Y=1, f(X_0) = -1)
       &\le (\tfrac{1}{2} - \epsilon) P(\BinModel(X_0) = -1, B=1)
       +  P(\BinModel(X_0)=-1,  Y=1, B=0) 
       \\
    P(Y=1, f(X_0) = -1, B=1)
       &\le (\tfrac{1}{2} - \epsilon) P(\BinModel(X_0) = -1, B=1) 
       \\
    P(Y=1  \mid f(X_0) = -1, B=1)
       &\le(\tfrac{1}{2} - \epsilon)P(Y=-1  \mid f(X_0) = -1, B=1) 
    \end{align*}
    The final inequality can be rewritten as
    \begin{align*}
        P(Y=-1  \mid f(X_0)=-1, B=1)
        \ge 
        (\tfrac{1}{2} + \epsilon)
    .\end{align*}
\end{proof}
%

\subsection{Proof of Theorem~\ref{th:risk-increase-surrogate}}

\RiskIncreaseSurrogate*

\begin{proof}
	Let $I = \ind\{f(X_0) = -1, B=1\}$ be the indicator for recourse in
	the negative class. Then, since $\varphi(X_0)$ lies on the decision
	boundary (DB) when $X_0$ is in the negative class,
	\begin{align*}
		R_Q(g)
		&= \E_{(X,Y) \sim Q}[\loss(g(X),Y)]\\
		&= \E_{(X_0,Y) \sim Q}[\loss(g(\varphi(X_0)),Y) I] 
            + \E_{(X,Y) \sim Q}[\loss(g(X),Y) (1-I)]\\
		&= \E_{(X_0, Y) \sim Q}[\loss(1/2,Y) I] 
        + \E_{(X_0, Y) \sim P}[\loss(g(X_0),Y) (1-I)]  
        \tag{$\varphi(X_0)$ on the DB}\\
		&= c P(f(X_0) = -1, B=1) +\E_{(X_0, Y) \sim P}[\loss(g(X_0),Y) (1-I)]  
        \tag{by definition of $c$}\\
        &\geq \E_P[\loss(g(X_0),Y) I] 
        + \E_{(X_0, Y) \sim P}[\loss(g(X_0),Y) (1-I)]\\
		&= R_P(g),
	\end{align*}
	where the inequality is equivalent to \eqref{eqn:good_surrogate_classifier}.
\end{proof}

\section{Additional results and proofs for Section~\ref{sec:strategic}}
\label{app:protecting}

\begin{figure*}[b]
    \begin{center}
        \includegraphics{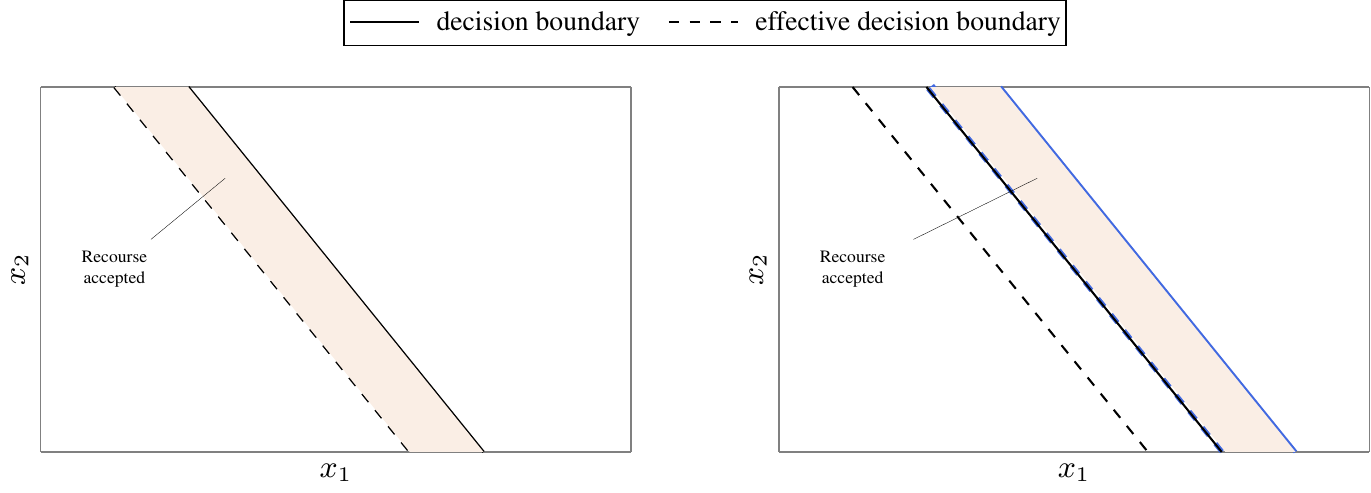}
        \caption{Left figure: Linear classifier with a shaded area to indicate
        where recourse is accepted. Right figure: The same linear classifier but shifted
        towards the right in such a way that the effective decision boundary is the original
        decision boundary.}
        \label{fig:strategic_classification}
    \end{center}
\end{figure*}

\subsection{Examples of classifiers invariant under recourse}

Let us justify more rigorously the linear classifier example introduced in
Section~\ref{sec:strategic}.  A visual representation describing this example
can be found in Figure~\ref{fig:strategic_classification}.

\begin{example}\label{ex:invariance_linear_classifiers}
Consider the set of linear classifiers $\modelclass =
	\{f_{\theta,\theta_0}(x) = \sign{x^\top \theta + \theta_0} \mid \theta
	\in \reals^d, \theta_0 \in \reals\}$ with the convention that
	$\sign{z} = +1$ for $z \geq 0$ and $\sign{z} = -1$ otherwise. If the
	recourse map is such that any point $x_0$ within distance $D > 0$ of
	the decision boundary of $f_{\theta,\theta_0}$ gets mapped to the
	positive class, then this class is invariant under recourse, because
	\[
	f_{\theta,\theta_0'}(\varphi(f_{\theta,\theta_0'},x_0))
	= f_{\theta,\theta_0}(x_0)
	\qquad
	\text{for all $x_0 \in \domain$}
	\]
	when $\theta_0' = \theta_0 - D \|\theta\|$. To see this, note that
	the (signed) distance from $x_0$ to the decision boundary for
	$f_{\theta,\theta_0'}$ is $\frac{-x_0^\top
		\theta-\theta_0'}{\|\theta\|}$. Hence the following are all
	equivalent:
	\begin{align*}
		f_{\theta,\theta_0'}(\varphi(f_{\theta,\theta_0'},x_0)) &= +1\\
		\frac{-x_0^\top\theta - \theta_0'}{\|\theta\|} &\leq D\\
		x_0^\top\theta + \theta_0' &\geq -D\|\theta\|\\
		x_0^\top\theta + \theta_0 &\geq 0\\
		f_{\theta,\theta_0}(x_0) &= +1.
	\end{align*}
\end{example}

One can extend this idea to other geometrical shapes: 

\begin{example}
Consider the spherical classifiers for which
	$f_{\theta,b}(x) = +1$ if and only if $\|x - \theta\| \geq b$. Then
	the set $\modelclass = \{f_{\theta,b} \mid \theta \in \reals^d, b \in
	\reals_+\}$ is invariant under recourse when the recourse map is again
	such that any point $x_0$ in the negative class that lies within
	distance $D > 0$ of the decision boundary of $f_{\theta,b}$ gets
	mapped to the positive class. This follows because providing recourse
	has the effect of effectively shrinking $b$ by $D$, so we can undo
	this effect by increasing $b$ to $b'= b+D$:
	\[
	f_{\theta,b'}(\varphi(f_{\theta,b'},x_0)) = f_{\theta,b}(x_0)
	\qquad
	\text{for all $x_0 \in \domain$.}
	\]
\end{example}


\subsection{Proof of Theorem~\ref{th:strategizing-defiant}}

\StrategizingDefiant*
\begin{proof}
Note that in the defiant case $Q(Y  \mid X_0) = P(Y  \mid X_0)$ as
\begin{align*}
    Q(Y  \mid X_0)
    &= \int\limits_{\mathcal{X}} Q(Y, X=\text{d}x  \mid X_0)\\
    &= \int\limits_{\mathcal{X}}
        Q(Y  \mid X=x, X_0 )Q(X=\text{d}x  \mid X_0)\\
    &= P(Y  \mid X_0) \int\limits_{\mathcal{X}} Q(X=\text{d}x \mid X_0)  \\
    &= P(Y  \mid X_0)
.\end{align*}
Using this we, we write
\begin{align*}
	\min_{f \in \modelclass} R_{Q_f}(f)
    &= \min_{f \in \modelclass}
        \E_{(X,Y) \sim Q_f}[\loss(f(X), Y)] \\
	&= \min_{f \in \modelclass}
        \E_{(X_0,Y) \sim Q_f}[\loss(f(\varphi_f(X_0)), Y)] \\
	&= \min_{f \in \modelclass}
    \E_{(X_0,Y) \sim P}[\loss(f(\varphi_f(X_0)), Y)] \tag{by defiant case}\\
	&= \min_{f \in \modelclass^{r}_\varphi}
        \E_{(X_0,Y) \sim P}[\loss(f(X_0), Y)] \tag{by definition of $\modelclass^{r}_\varphi$} \\
	&= \min_{f \in \modelclass}
        \E_{(X_0,Y) \sim P}[\loss(f(X_0), Y)] \tag{since $\modelclass=\modelclass^{r}_\varphi$} \\
        \min_{f \in \modelclass} R_{Q_f}(f)	&= \min_{f \in \modelclass} R_P(f)
\, .
\end{align*}
\end{proof}
%

\subsection{Explicit $\Delta$ bound}

\begin{example}
	Let us further specialize the setting of
	Example~\ref{ex:invariance_linear_classifiers} to the task of
	distinguishing between two Gaussians with different means
	$\mu,\nu \in \reals^d$ and common positive definite
	covariance matrix~$\Sigma$. That is, let $P(Y=-1) = P(Y=+1) = 1/2$, 
    and the data is distributed according to 
	$P(X_0|Y=+1) = \normal(\mu,\Sigma)$ and $P(X_0|Y=-1) = \normal(\nu,\Sigma)$. 
	Then $ f = f_{\theta,\theta_0}$
	for $\theta = \Sigma^{-1}(\mu- \nu)$ and $\theta_0 =
	-\tfrac{1}{2}(\mu + \nu)^\top \Sigma^{-1}(\mu -
    \nu)$ is the Bayes optimal classifier. See Section~\ref{sec:additional_gaus_details}
	for a justification. The compensating classifier is given by $ f' =
	f_{\theta,\theta_0'}$ for $\theta_0' = \theta_0 - D\|\theta\|$. Then
	recourse for $ f'$ affects users in a band of width $D\|\theta\|$
	which lies just in the positive class according to $ f$:
	\[
	\mathcal{A} = \{ x : 0 \leq x^\top \theta + \theta_0 < D\|\theta\| \},
	\]
	and
	\begin{align*}
		\Delta
		&= \int_{\mathcal{A}} \Big\{P(Y=-1 \mid X_0=x_0)
        - P(Y=-1 \mid X_0=\varphi_{ f'}(x_0))\Big\} 
            P(\text{d}x_0)\\
		&= \int_{\mathcal{A}} \Big\{P(Y=-1 \mid X_0=x_0)
        - P\left(Y=-1 \mid X_0=x_0 - \frac{ f'(x_0)}{\|\theta\|^2} \theta\right) \Big\} 
            P(\text{d}x_0).
	\end{align*}
    Again by Section~\ref{sec:additional_gaus_details}, we can write the posterior distribution as
    \begin{align*}
        P(Y=-1 | X_0 =x) =  \frac{1}{1 + e^{\theta^{\top}x + \theta_0}}
    ,\end{align*}
    Now, we write
    \begin{align*}
        \theta^\top \varphi( f',x_0) + \theta_0
        &= \theta^\top x_0 + \theta_0
            - (\theta)^\top 
            \left(\frac{ f'(x_0)}{\norm{\theta}^2}\theta \right) \\
        &= \theta^\top x_0 + \theta_0 + 1 
    \, ,
    \end{align*}
    since $ f'(x_0)=-1$. 
    Thus, $\theta^\top \varphi_{ f'}(x_0) + \theta > \theta^\top x_0 + \theta_0$. 
    Since the mapping $t\mapsto 1/(1+\exp(t))$ is decreasing,
    we deduce that $\Delta >0$ in this case. 
\end{example}


\subsection{Proof of Theorem~\ref{th:strategizing-compliant}}

\StrategizingCompliant*
\begin{proof}
As in the defiant case, the proof considers the strategic
choice $ f'$, which compensates for the effect of recourse in order
to maintain the same decision boundary as in the case without recourse.
The effect of recourse is then only to change the distribution in a way
that is captured by Definition~\ref{def:compliant_move}.
We first notice that $\min_{f \in \modelclass} R_{Q_f}(f) \leq R_{Q_{ f'}}( f')$
holds because $ f' \in \modelclass$. 
We write
\begin{align*}
	R_{Q_{ f'}}( f')
	&= \E_{(X,Y) \sim Q_{ f'}}[\loss( f'(X), Y))] \\
	&= \E_{(X_0,Y) \sim Q_{ f'}}[\loss( f'(\varphi_{ f'}(X_0)), Y)] \\
	&= \E_{(X_0,Y) \sim Q_{ f'}}[\loss( f(X_0), Y)]\\
	&= \E_{(X_0,Y) \sim P}[\loss( f(X_0), Y)] - \Delta \tag{definition of $\Delta$} \\
R_{Q_{ f'}}( f')	&= \min_{f \in \modelclass} R_P(f) - \Delta
	\, . \end{align*} \end{proof}
%


\section{Details Experimental Setup}
\label{app:experiments}
All the code to reproduce the experiments and figures in this paper can be found 
in a GitHub repository.\footnote{\href{
    https://github.com/HiddeFok/consequences-of-recourse}{
github.com/HiddeFok/consequences-of-recourse
}}
All experiments were performed on a single 
$32$-core CPU node (AMD Rome 7H12) with $64$GB of RAM.  
For all experiments we 
used the following classifiers from the \verb|scikit-learn| (version $1.0$) library:
\begin{itemize}
    \item \verb|LogisticRegression|, with the default parameters,
        except for 
        \verb|class_weigt='balanced'|.
    \item \verb|GradientBoostingClassifier|, with the default parameters,
        except for
        \verb|n_estimators=10| .
    \item \verb|DecisionTreeClassifier|, with the default parameters,
        except for 
        \verb|class_weigt='balanced'| and \verb|max_depth=4|. 
    \item \verb|GaussianNB|, with the default parameters.
    \item \verb|RandomForestClassifier|, with the default parameters
        except for 
        \verb|class_weigt='balanced'|,
        \verb|max_depth=4| 
        and \verb|n_estimators=10|. 
   \item \verb|QuadraticDiscrimantAnalysis|, with the default parameters.
   \item \verb|MLPClassifier|, with the default parameters
        and hidden layers, \\
        \verb|hidden_layer_sizes=(4,)|. 
   \item \verb|MLPClassifier|, with the default parameters
        and hidden layers, \\
        \verb|hidden_layer_sizes=(4, 4)|. 
   \item \verb|MLPClassifier|, with the default parameters
        and hidden layers, \\
        \verb|hidden_layer_sizes=(8,)|. 
   \item \verb|MLPClassifier|, with the default parameters
        and hidden layers, \\
        \verb|hidden_layer_sizes=(8, 16)|. 
   \item \verb|MLPClassifier|, with the default parameters
        and hidden layers,\\
        \verb|hidden_layer_sizes=(8, 16, 8)|. 
\end{itemize}
For both the synthetic data experiments and real data experiments, 
we repeated the experiments $10$ times to estimate confidence bounds. 
The result with the lower risk is written in bold in the result tables. 
If the confidence bounds overlap we make both results bold.

\subsection{Synthetic Data}
For each of the synthetic data experiments, we generated $5000$ training samples and
$1000$ test samples with balanced classes, i.e.
$P(Y=+1) = P(Y=-1)=\tfrac{1}{2}$. Examples of the results of these experiments 
can be found in Figures~\ref{fig:moons_examples}, \ref{fig:circles_examples}, and
\ref{fig:gaussians_examples}.

\paragraph{Moons dataset}
The Moons dataset is acquired through the \verb|make_moons| function from
\verb|scikit-learn|.  A data point $X_0$ is sampled by first discretizing  
$[0, \pi)$ uniformly and drawing one of these points $U$ uniformly, without replacement. Then,
sample $\epsilon \sim \normal(0, \sigma ^2 I_2)$ and construct $X_0$ by setting
\begin{align*}
    X_0  
    &\coloneqq s(U) + \epsilon 
    = \begin{bmatrix} \cos(U) \\ \sin(U) \end{bmatrix} 
        + \epsilon 
      & &\text{for } Y=+1\\
    X_0 
      &\coloneqq c(U) + \epsilon
      = \begin{bmatrix} 1 - \cos(U) \\ 1 - \sin(U)  \end{bmatrix} 
        - \begin{bmatrix} 0 \\ \tfrac{1}{2} \end{bmatrix} 
        + \epsilon
      & &\text{for } Y=-1
.\end{align*}
For our examples we selected $\sigma = 0.2$. 

To re-sample the label $Y$ after providing
recourse we calculate the conditional distribution of this model. Let, $p$ be the 
density of $\epsilon$. Then, the density
of $X_0  \mid  Y = +1$ and $X_0 \mid  Y= -1$, denoted by $g_+$ and $g_{-}$ respectively, 
is given by
\begin{align*}
    g_{+}(x_0) 
        &= \frac{1}{\pi} \int_{0}^{\pi} p(x_0 - s(u))  \text{d}u,\\
    g_{-}(x_0) 
        &= \frac{1}{\pi} \int_{0}^{\pi} p(x_0 - c(u))  \text{d}u
.\end{align*}
In our implementation this integral is approximated by a Riemann sum. 
The conditional distribution now follows,
\begin{align*}
    P(Y=1  \mid X_0 = x_0) 
    &= \frac{g_{+}(x_0)P(Y=1) }{g_{+}(x_0)P(Y=1) + g_{-}(x_0) P(Y=-1) } \\
    &= \frac{g_{+}(x_0) }{g_{+}(x_0) + g_{-}(x_0) }
.\end{align*}
\paragraph{Circles dataset}
The Circles dataset is acquired through the \verb|make_circles| function from
\verb|scikit-learn|. A data point $X$ is sampled by first discretizing  
$[0, 2\pi)$ uniformly and drawing one of these points $U$ uniformly, without replacement. 
Then, sample $\epsilon \sim \normal(0, \sigma ^2 I_2)$, set $\lambda \in (0, 1)$
and construct $X$ by setting
\begin{align*}
    X_0  
    &\coloneqq \lambda s(U) + \epsilon 
    = \lambda \begin{bmatrix} \cos(U) \\ \sin(U) \end{bmatrix} 
        + \epsilon 
      & &\text{for } Y=+1\\
    X_0 
      &\coloneqq s(U) + \epsilon
      = \begin{bmatrix} \cos(U) \\ \sin(U) \end{bmatrix} 
        + \epsilon
      & &\text{for } Y=-1
.\end{align*}
For our examples we selected $\sigma = 0.2$ and $\lambda=0.6$. 

To re-sample the label $Y$ after providing
recourse we calculate the conditional distribution of this model. Let, $p$ be the 
density of $\epsilon$. Then, the density
of $X_0  \mid  Y = 1$ and $X_0  \mid  Y= -1$, denoted by $g_+$ and $g_{-}$ respectively, 
is given by
\begin{align*}
    g_{+}(x_0) 
        &= \frac{1}{2\pi} \int_{0}^{2\pi} p(x_0 - \lambda s(u))  \text{d}u,\\
    g_{-}(x_0) 
        &= \frac{1}{2\pi} \int_{0}^{2\pi} p(x_0 - s(u))  \text{d}u
.\end{align*}
In our implementation this integral is approximated by a Riemann sum. 
The conditional distribution now follows,
\begin{align*}
    P(Y=1  \mid X_0 = x_0) 
    &= \frac{g_{+}(x_0)P(Y=1) }{g_{+}(x_0)P(Y=1) + g_{-}(x_0) P(Y=-1) } \\
    &= \frac{g_{+}(x_0) }{g_{+}(x_0) + g_{-}(x_0) }
.\end{align*}
\begin{table}[b]
    \caption{Estimated risks on the Credit dataset.
    Lower risk is bold.}
    \label{tbl:real_data_credit}
    \centering 
    \footnotesize
    \setlength\tabcolsep{4pt}
    \begin{tabular}{m{1cm}|cccccc}
        \toprule
        &
        \multicolumn{2}{c}{Wachter} & 
        \multicolumn{2}{c}{GS} & 
        \multicolumn{2}{c}{CoGS}\\
        \hline
        & 
        $R_P$ & $R_Q$ & $R_P$ & $R_Q$ & $R_P$ & $R_Q$\\
        \toprule
        \bottomrule
        LR & 0.17 $\pm$ 0.04& \textbf{0.05} $\pm$ \textbf{0.02}
        & 0.17 $\pm$ 0.04& \textbf{0.05} $\pm$ \textbf{0.01}
        & 0.17 $\pm$ 0.03& \textbf{0.05} $\pm$ \textbf{0.01}
        \\

        GBT & \textbf{0.07} $\pm$ \textbf{0.00}& \textbf{0.07} $\pm$ \textbf{0.00}
        & \textbf{0.07} $\pm$ \textbf{0.00}& \textbf{0.07} $\pm$ \textbf{0.00}
        & \textbf{0.07} $\pm$ \textbf{0.00}& \textbf{0.07} $\pm$ \textbf{0.00}
        \\

        DT & \textbf{0.27} $\pm$ \textbf{0.09}& \textbf{0.12} $\pm$ \textbf{0.18}
        & 0.23 $\pm$ 0.13& \textbf{0.05} $\pm$ \textbf{0.01}
        & 0.23 $\pm$ 0.11& \textbf{0.06} $\pm$ \textbf{0.01}
        \\

        NB & \textbf{0.18} $\pm$ \textbf{0.27}& \textbf{0.06} $\pm$ \textbf{0.01}
        & \textbf{0.16} $\pm$ \textbf{0.17}& \textbf{0.06} $\pm$ \textbf{0.01}
        & \textbf{0.18} $\pm$ \textbf{0.37}& \textbf{0.06} $\pm$ \textbf{0.01}
        \\

        QDA & \textbf{0.21} $\pm$ \textbf{0.39}& \textbf{0.07} $\pm$ \textbf{0.01}
        & \textbf{0.14} $\pm$ \textbf{0.14}& \textbf{0.06} $\pm$ \textbf{0.01}
        & \textbf{0.13} $\pm$ \textbf{0.15}& \textbf{0.07} $\pm$ \textbf{0.02}
        \\

        NN 1 & \textbf{0.06} $\pm$ \textbf{0.01}& \textbf{0.06} $\pm$ \textbf{0.01}
        & \textbf{0.06} $\pm$ \textbf{0.01}& \textbf{0.06} $\pm$ \textbf{0.01}
        & \textbf{0.06} $\pm$ \textbf{0.01}& \textbf{0.06} $\pm$ \textbf{0.00}
        \\

        NN 2 & \textbf{0.07} $\pm$ \textbf{0.01}& \textbf{0.06} $\pm$ \textbf{0.01}
        & \textbf{0.07} $\pm$ \textbf{0.01}& \textbf{0.07} $\pm$ \textbf{0.01}
        & \textbf{0.07} $\pm$ \textbf{0.01}& \textbf{0.07} $\pm$ \textbf{0.01}
        \\

        NN 3 & \textbf{0.07} $\pm$ \textbf{0.01}& \textbf{0.07} $\pm$ \textbf{0.01}
        & \textbf{0.07} $\pm$ \textbf{0.00}& \textbf{0.07} $\pm$ \textbf{0.00}
        & \textbf{0.07} $\pm$ \textbf{0.01}& \textbf{0.07} $\pm$ \textbf{0.00}
        \\

        NN 4 & \textbf{0.06} $\pm$ \textbf{0.01}& \textbf{0.06} $\pm$ \textbf{0.01}
        & \textbf{0.06} $\pm$ \textbf{0.01}& \textbf{0.06} $\pm$ \textbf{0.01}
        & \textbf{0.06} $\pm$ \textbf{0.00}& \textbf{0.07} $\pm$ \textbf{0.00}
        \\

        NN 5 & \textbf{0.06} $\pm$ \textbf{0.01}& \textbf{0.06} $\pm$ \textbf{0.01}
        & \textbf{0.06} $\pm$ \textbf{0.01}& \textbf{0.06} $\pm$ \textbf{0.01}
        & \textbf{0.06} $\pm$ \textbf{0.00}& \textbf{0.07} $\pm$ \textbf{0.00}
        \\
    \end{tabular}
\end{table}
\begin{table}[t]
    \caption{Estimated risks on the HELOC dataset.
    Lower risk is bold.}
    \label{tbl:real_data_heloc}
    \centering 
    \footnotesize
    \setlength\tabcolsep{4pt}
    \begin{tabular}{m{1cm}|cccccc}
        \toprule
        &
        \multicolumn{2}{c}{Wachter} & 
        \multicolumn{2}{c}{GS} & 
        \multicolumn{2}{c}{CoGS}\\
        \hline
        & 
        $R_P$ & $R_Q$ & $R_P$ & $R_Q$ & $R_P$ & $R_Q$\\
        \toprule
        \bottomrule
        LR & \textbf{0.27} $\pm$ \textbf{0.03}& 0.40 $\pm$ 0.04
        & \textbf{0.28} $\pm$ \textbf{0.02}& 0.42 $\pm$ 0.02
        & \textbf{0.27} $\pm$ \textbf{0.04}& 0.45 $\pm$ 0.02
        \\

        GBT & \textbf{0.19} $\pm$ \textbf{0.02}& \textbf{0.21} $\pm$ \textbf{0.04}
        & \textbf{0.20} $\pm$ \textbf{0.02}& \textbf{0.25} $\pm$ \textbf{0.12}
        & \textbf{0.20} $\pm$ \textbf{0.01}& 0.35 $\pm$ 0.05
        \\

        DT & \textbf{0.19} $\pm$ \textbf{0.01}& \textbf{0.29} $\pm$ \textbf{0.31}
        & \textbf{0.20} $\pm$ \textbf{0.02}& \textbf{0.25} $\pm$ \textbf{0.13}
        & \textbf{0.20} $\pm$ \textbf{0.02}& 0.35 $\pm$ 0.09
        \\

        NB & \textbf{0.29} $\pm$ \textbf{0.02}& 0.45 $\pm$ 0.03
        & \textbf{0.29} $\pm$ \textbf{0.02}& 0.45 $\pm$ 0.07
        & \textbf{0.28} $\pm$ \textbf{0.03}& 0.51 $\pm$ 0.05
        \\

        QDA & \textbf{0.32} $\pm$ \textbf{0.03}& 0.47 $\pm$ 0.03
        & \textbf{0.32} $\pm$ \textbf{0.02}& 0.49 $\pm$ 0.03
        & \textbf{0.31} $\pm$ \textbf{0.03}& 0.52 $\pm$ 0.03
        \\

        NN 1 & \textbf{0.27} $\pm$ \textbf{0.03}& 0.46 $\pm$ 0.03
        & \textbf{0.28} $\pm$ \textbf{0.02}& 0.46 $\pm$ 0.03
        & \textbf{0.28} $\pm$ \textbf{0.02}& 0.49 $\pm$ 0.02
        \\

        NN 2 & \textbf{0.30} $\pm$ \textbf{0.12}& 0.47 $\pm$ 0.02
        & \textbf{0.27} $\pm$ \textbf{0.03}& 0.45 $\pm$ 0.04
        & \textbf{0.30} $\pm$ \textbf{0.12}& 0.51 $\pm$ 0.07
        \\

        NN 3 & \textbf{0.28} $\pm$ \textbf{0.03}& 0.46 $\pm$ 0.04
        & \textbf{0.27} $\pm$ \textbf{0.02}& 0.46 $\pm$ 0.02
        & \textbf{0.27} $\pm$ \textbf{0.03}& 0.50 $\pm$ 0.02
        \\

        NN 4 & \textbf{0.27} $\pm$ \textbf{0.02}& 0.44 $\pm$ 0.05
        & \textbf{0.27} $\pm$ \textbf{0.04}& 0.45 $\pm$ 0.03
        & \textbf{0.26} $\pm$ \textbf{0.02}& 0.48 $\pm$ 0.04
        \\

        NN 5 & \textbf{0.27} $\pm$ \textbf{0.02}& 0.44 $\pm$ 0.05
        & \textbf{0.27} $\pm$ \textbf{0.04}& 0.45 $\pm$ 0.03
        & \textbf{0.26} $\pm$ \textbf{0.02}& 0.48 $\pm$ 0.04
        \\
    \end{tabular}
\end{table}

\paragraph{Gaussians dataset}
The Gaussian data points are sampled from $2$ Gaussians with different means 
$\mu, \nu \in \mathbb{R}^{2}$ and different
covariances $\Sigma_{+}, \Sigma_{-} \in \mathbb{R}^{2 \times 2}$, %
\begin{align*}
    X  \mid Y &= +1 \sim \normal(\mu, \Sigma_{+})\\
    X  \mid Y &= -1 \sim \normal(\nu, \Sigma_{+})
.\end{align*}
The densities of the $2$ conditional distributions are given by
\begin{align*}
    g_+(x_0)
    &= \frac{1}{2 \pi |\Sigma_+|}
    e^{-\frac{1}{2}(x_0 - \mu)^{\top}\Sigma^{-1}_+(x_0 - \mu)} \\   
    g_0(x_0)
    &= \frac{1}{2 \pi |\Sigma_-|}
    e^{-\frac{1}{2}(x_0 - \nu)^{\top}\Sigma^{-1}_-(x_0 - \nu)}
.\end{align*}

The conditional distribution is given by the odds between
the densities of the $2$ gaussians
\begin{align*}
    P(Y=1 | X_0=x_0) = \frac{g_{+}(x_0)}{g_{-}(x_0) + g_{+}(x_0)}
.\end{align*}

\subsubsection{Linear relation between $p$ and $R_Q$}
\label{app:linear_relation}

Here, we derive the linear relation seen in
Figure~\ref{fig:risk_p_sigma_increase}. Denote by $R_Q(p)$ the risk after
recourse dependent on $p$, then the following relation can be calculated
\begin{align}\label{eq:linear_relation}
    R_Q(p) = R_P(f) + p(R_Q(1)- R_P(f))
.\end{align}
We can apply Lemma~\ref{lem:express_Q_risk} and use that $r(x_0) = p$ 
for all $x_0$ to get the following expressions:
\begin{align*}
    R_Q(p) 
    &= \mathbb{E}\left[r(X_0)p_{-}(\varphi(X_0))\right] 
    + 
    \mathbb{E}\left[ (1 - r(X_0)) 1 \{f(X_0) \neq Y\}   \right] \\
    &= p\mathbb{E}\left[ p_{-}(\varphi(X_0))\right] 
    + 
    (1 - p) \mathbb{E}\left[1 \{f(X_0) \neq Y\}   \right] \\
    &= p\mathbb{E}\left[ p_{-}(\varphi(X_0))\right]
    + 
    (1 - p) R_P
.\end{align*}
This also gives us that $R_Q(1) = \mathbb{E}\left[ p_{-}(\varphi(X_0))\right]$. 
Substituting and rewriting gives 
\begin{align*}
    R_Q(p)
    &= p R_Q(1) + (1 - p) R_P \\
    &= R_P + p(R_Q(1) - R_P) 
.\end{align*}
\subsection{Real Data}
Here, we describe how the experiments for the real data were performed.
\paragraph{Conditional distribution estimation}
As mentioned before the main challenge with real data is that we do not have access
to $P(Y  \mid X_0 )$. To circumvent this, we estimate this function as well as possible, 
by reserving most of the data to train a calibrated classifier. $N_{\text{cond train}}$ are
used to train this classifier and $N_{\text{cond calib}}$ are used to calibrate
this classifier. The exact values of the data splits
are given in Table~\ref{tbl:real_data_params}. Furthermore, 
we perform a grid search over a large set of parameters using cross validation
to find the best performing calibrated classifier. The parameters in the 
grid search are 
\begin{itemize}
    \item \verb|learning_rate|: $\{0.05, 0.15\} $,
    \item \verb|n_estimators|: $\{10, 20, 60\} $,
    \item \verb|subsample|: $\{0.8, 0.9, 1\} $,
    \item \verb|max_depth|: $\{1, 2, 3\} $.
\end{itemize}
As a base classifier
we use the \texttt{Gradient} \texttt{BoostedClassifier} from \verb|scikit-learn| and
we use Platt scaling \citep{platt1999probabilistic} to calibrate 
the probabilities.

\paragraph{Classification and Recourse} After a conditional distribution is
estimated for each dataset, we train the same set of classifiers as for the synthetic data 
on $N_{\text{train}}$ different
data points. Then, counterfactuals are generated using the different methods and
using the trained conditional estimated distribution a new class label is sampled 
for the position at the counterfactual point. The estimated risk is then calculated
for the dataset before and after recourse is provided.

\begin{table}[b]
    \caption{Details of the datasets used during the experiments}
    \label{tbl:real_data_params}
    \centering 
    \setlength\tabcolsep{4pt}
    \begin{tabular}{m{2cm}ccc}
        \toprule
        & 
        Credit data & 
        Adult data & 
        HELOC data \\
        \toprule
        $P(Y=+1)$ & 0.932 & 0.239 & 0.480 \\
        $P(Y=-1)$ & 0.068 & 0.761 & 0.520 \\
        $N_{\text{cond train}}$ & 40000 & 30000 & 5000 \\
        $N_{\text{cond calib}}$ & 10000 & 10000 & 2000 \\
        $N_{\text{train}}$ & 5000 & 5000 & 5000 \\
        $N_{\text{test}}$ & 1000 & 1000 & 1000 \\
        \bottomrule
    \end{tabular}
\end{table}

\newpage
\begin{figure}[h]
    \centering
    \includegraphics[scale=0.8]{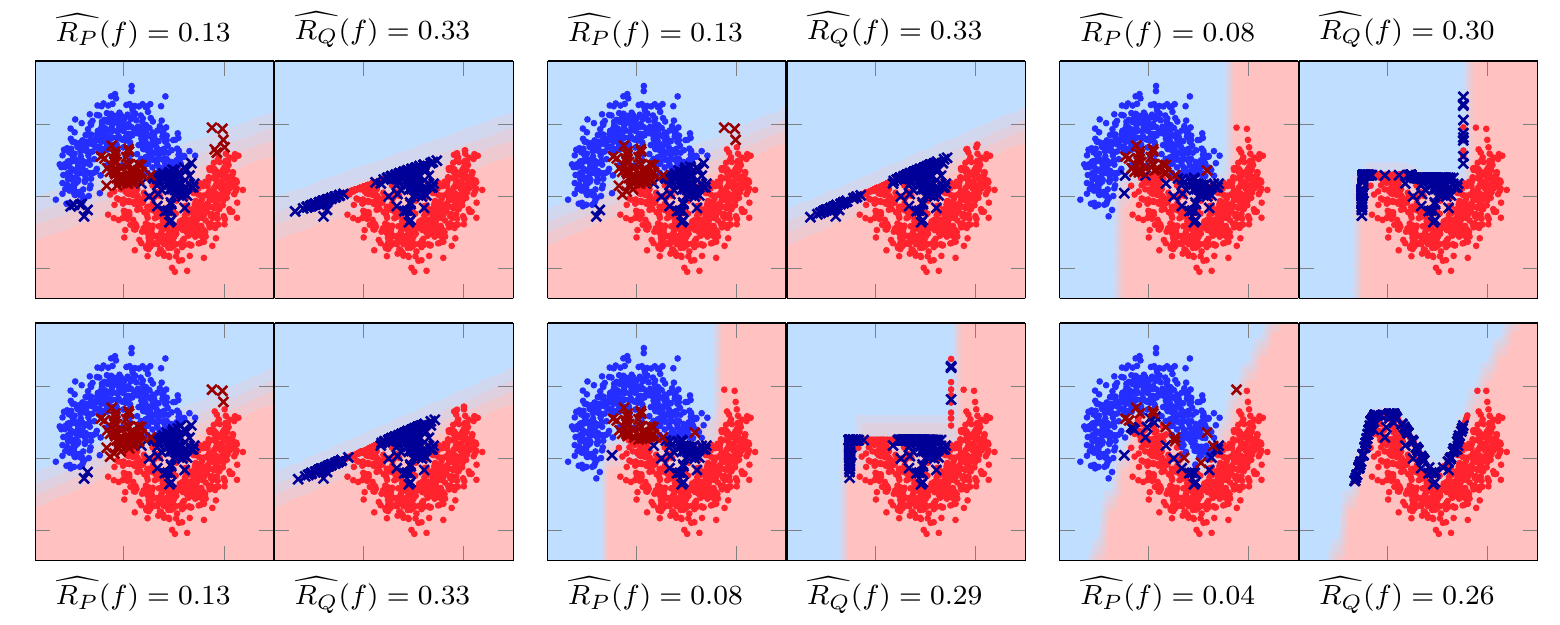}
    \caption{Examples of the effect of giving recourse with various classifiers on
    the Moons data set. From left to Right, Top to Bottom: 
    LR, QDA, GBT, NB, DT, NN(8, 16).}
    \label{fig:moons_examples}
\end{figure}

\begin{figure}[h]
    \centering
    \includegraphics[scale=0.8]{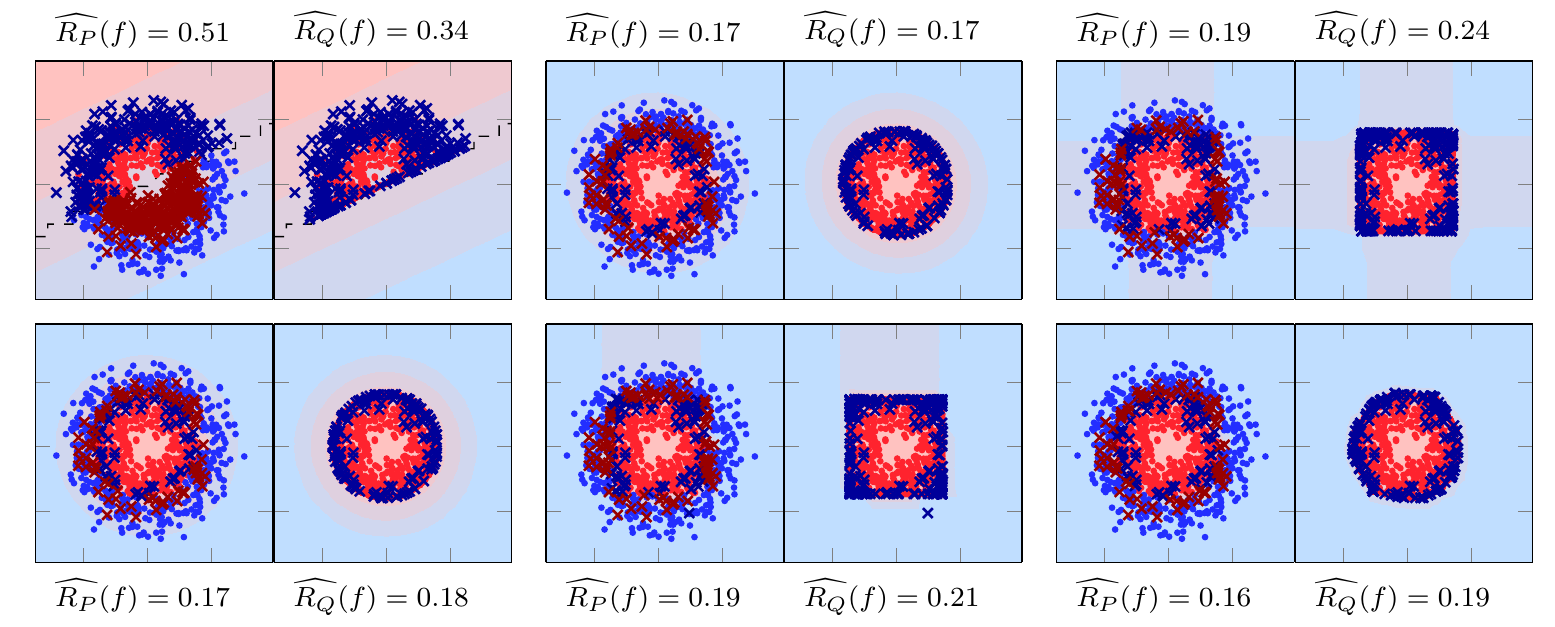}
    \caption{Examples of the effect of giving recourse with various classifiers on
    the Circles data set. From left to Right, Top to Bottom: 
    LR, QDA, GBT, NB, DT, NN(8, 16).}
    \label{fig:circles_examples}
\end{figure}

\begin{figure}[h]
    \centering
    \includegraphics[scale=0.8]{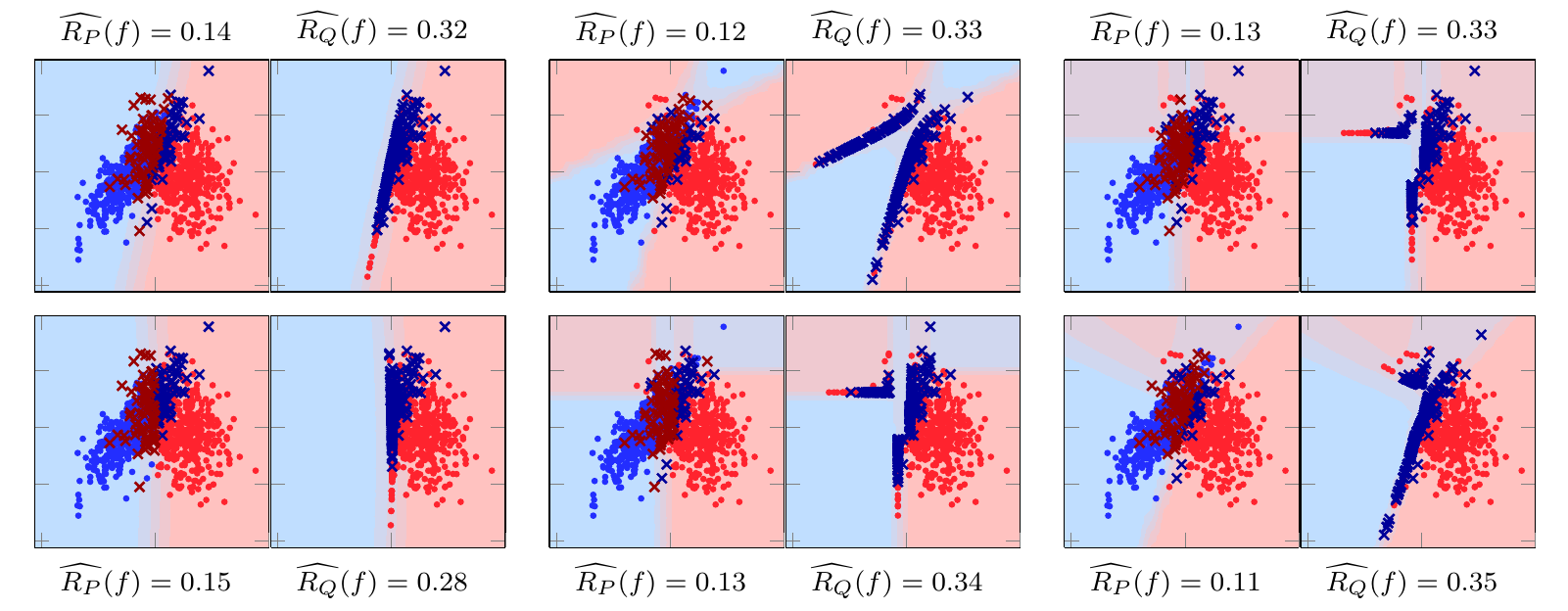}
    \caption{Examples of the effect of giving recourse with various classifiers on
    the Gaussians data set. From left to Right, Top to Bottom: 
    LR, QDA, GBT, NB, DT, NN(8, 16).}
    \label{fig:gaussians_examples}
\end{figure}

\end{document}